\documentclass{article}
\usepackage{arxiv}
\usepackage[numbers]{natbib}
\usepackage{float}
\usepackage{placeins}
\usepackage{booktabs}    % 三线表
\usepackage{multirow}    % 合并单元格
\usepackage{pgfplots}    % 柱状图
\pgfplotsset{compat=1.18}
\usepackage{graphicx}
\usepackage{bbm}
\usepackage{listings}
\usepackage[most]{tcolorbox}
\usepackage{makecell}

\usepackage{amsmath, amsthm, amssymb}
\usepackage{algorithm}
\usepackage{algpseudocode}
\usepackage[shortlabels]{enumitem}
\usepackage{subcaption}

\usepackage{tikz}
\usetikzlibrary{shapes.geometric, arrows.meta, positioning, fit, backgrounds}

\usepackage[utf8]{inputenc} % allow utf-8 input
\usepackage[T1]{fontenc}    % use 8-bit T1 fonts
\usepackage{hyperref}       % hyperlinks
\usepackage{url}            % simple URL typesetting
\usepackage{amsfonts}       % blackboard math symbols
\usepackage{nicefrac}       % compact symbols for 1/2, etc.
\usepackage{microtype}      % microtypography
\usepackage{xcolor}    % colors

\newcommand{\best}[1]{\textbf{\textcolor{black}{#1}}}

\theoremstyle{definition}

\theoremstyle{remark}

\newcommand{\Dcal}{\mathcal{D}_{\mathrm{cal}}}
\newcommand{\Acal}{\mathcal{A}}

\newcommand{\Xcal}{\mathcal{X}}
\newcommand{\Ycal}{\mathcal{Y}}

\title{CorePath: A Breast-Specialized Pathology Foundation Model for Core Needle Biopsy Diagnosis and Risk-Controlled Report Generation}

\author{%
\textbf{Ting Yin}$^{1,2}$\thanks{Equal contribution.} \quad
\textbf{Danning Li}$^{3}$\footnotemark[1] \quad
\textbf{Chen Shu}$^{4}$ \quad
\textbf{Xiaoxia Yao}$^{1}$ \quad
\textbf{Boyu Fu}$^{5}$ \quad
\textbf{Yujing Chang}$^{1}$ \\
\textbf{Tianyu Shi}$^{4}$ \quad
\textbf{Mengna Feng}$^{1}$ \quad
\textbf{Jie Chen}$^{2}$ \quad
\textbf{Jing Fu}$^{6}$ \quad
\textbf{Xiuli Xiao}$^{7,8}$ \quad
\textbf{Tianlin Li}$^{7}$ \\
\textbf{Mumin Shao}$^{9}$ \quad
\textbf{Jiaxin Bi}$^{9}$ \quad
\textbf{Wenchuan Zhang}$^{1,2}$ \quad
\textbf{Xiaoyan Wu}$^{1,2}$ \quad
\textbf{Xiao Han}$^{10}$\thanks{Corresponding authors.} \\
\textbf{Zhang Zhang}$^{1,2}$\footnotemark[2] \quad
\textbf{Yuhao Yi}$^{1,2,4}$\footnotemark[2] \quad
\textbf{Hong Bu}$^{1,2}$ \\
{
\begin{minipage}{\dimexpr\textwidth-2\tabcolsep-1em\relax}
\centering
$^1$Department of Pathology, West China Hospital, Sichuan University, Chengdu, China; 
$^2$Institute of Clinical Pathology, West China Hospital, Sichuan University, Chengdu, China;
$^3$The Hong Kong University of Science and Technology (Guangzhou); 
$^4$College of Computer Science, Sichuan University; 
$^5$Sichuan University-Pittsburgh Institute, Sichuan University; 
$^6$Department of Pathology, Sichuan Provincial People's Hospital, Chengdu, China; 
$^7$Department of Pathology, The Affiliated Hospital of Southwest Medical University, Luzhou, China; 
$^8$Department of Pathology, The Fourth Affiliated Hospital of Southwest Medical University, Meishan, China; 
$^9$Department of Pathology, Shenzhen Traditional Chinese Medicine Hospital, The Fourth Clinical Medical College of Guangzhou University of Chinese Medicine, Shenzhen, China;
$^{10}$College of Biomedical Engineering, Sichuan University, Chengdu, China\\
\texttt{yuhaoyi@scu.edu.cn, zhangzhang714@163.com, xiao\_han@scu.edu.cn}
\end{minipage}}
}

\date{} % Do not add an automatic compilation date to the original manuscript.
\renewcommand{\shorttitle}{CorePath for Breast CNB Diagnosis and Reporting}
\begin{document}

\maketitle
% \nolinenumbers

\begin{abstract}
Breast core needle biopsy (CNB) is central to breast cancer diagnosis yet remains challenging because limited tissue sampling, lesion heterogeneity, and subtle morphologic overlap can obscure subtype distinctions. We developed \textbf{CorePath}, a breast-specialized multimodal pathology foundation model fine-tuned from PRISM using 7901 paired CNB whole-slide images and diagnostic reports from two centers. Evaluated across six CNB cohorts and two public breast pathology benchmarks without task-specific retraining, CorePath consistently outperformed PRISM across cancer detection, invasion assessment, and histological subtyping. It achieved weighted area under the receiver operating characteristic curves (AUCs) of 0.9526-0.9735 for five-class CNB histological subtyping across private centers. On public benchmarks, CorePath outperformed leading pathology foundation models, achieving the highest weighted AUCs of 0.7780 for BCNB invasive carcinoma subtyping, 0.8178 for BRACS lesion stratification, and 0.8252 for BRACS fine-grained classification. In report generation, CorePath reduced the overall non-breast hallucinations from 30.1\% to 2.8\%, demonstrating improved domain fidelity after breast-specific adaptation. \textbf{CorePath-CRG} further combined conformal filtering of subtype and binary cancer status predictions with Learn-Then-Test-based threshold calibration to support selective narrative release, diagnostic fallback, and deferral. CorePath-CRG achieved zero non-breast hallucinations among released outputs and showed the strongest overall performance in pathologist-validated LLM-based Evaluation Scores and quantitative report-generation metrics across most centers. These results demonstrate that domain-specialized foundation models with statistical risk control offer a promising approach for accurate breast CNB diagnosis and reliable report generation.

\end{abstract}

\section{Introduction}
Breast cancer is the most commonly diagnosed malignancy and a leading cause of cancer-related death among women worldwide~\cite{gradishar2024breast,Bray2024GlobalCS}. Percutaneous core needle biopsy (CNB) is widely used as the initial diagnostic procedure for breast cancer, particularly for non-palpable lesions~\cite{yang2023histological}. Despite its central role in clinical management, CNB remains a diagnostically challenging setting~\cite{collins2021precision,shaaban2025diagnostic}. Unlike surgical excision specimens, CNB provides limited tissue and may incompletely sample heterogeneous lesions, thereby reducing diagnostically useful morphologic context~\cite{bilous2010breast}. Interpretation of CNB is further complicated by subtle morphologic overlap among diagnostically adjacent entities and by cross-institutional variation in tissue handling, staining, scanning, and reporting~\cite{quintana2018assessing,wang2025self}. Consequently, breast CNB interpretation is expertise-intensive and contributes substantial workload in routine pathology services~\cite{shaaban2025diagnostic}.

Recent advances in pathology foundation models and multimodal generative systems have created new opportunities for artificial intelligence (AI)-assisted diagnosis and report drafting from whole-slide images (WSIs)~\cite{huang2023visual, lu2024avisionlanguage, shaikovski2024prism}. By learning transferable visual-language representations from large-scale histopathology data, these models have shown encouraging performance in slide-level understanding and related downstream tasks. However, their direct application to routine breast CNB diagnosis remains challenging. Existing pathology foundation models are generally developed for broad pathology applications rather than for the specialized diagnostic setting of breast CNB. Moreover, although multimodal pathology systems can generate coherent free-text outputs, their clinical adoption in high-risk diagnostic workflows requires both reliable uncertainty estimation and the ability to defer cases that cannot be assessed with sufficient confidence~\cite{olsson2022estimating}.

Fine-tuning refers to adapting a pretrained foundation model to a target clinical task using domain-specific supervision, while parameter-efficient approaches aim to retain broad pretrained knowledge by updating only a small subset of model parameters~\cite{Sung_2022_CVPR}. Paired WSI-report data provide a natural source of supervision by linking tissue morphology with case-level diagnostic interpretation~\cite{ding2025multimodal}. Motivated by these considerations, we fine-tuned PRISM using parameter-efficient fine-tuning on paired breast WSIs and reports to develop \textbf{CorePath}, a breast CNB pathology foundation model designed to improve the alignment between slide-level histomorphologic representations and downstream breast CNB diagnosis while preserving the general capabilities of the pretrained model~\cite{shaikovski2024prism}.

Importantly, improved alignment and stronger predictive performance alone are not sufficient for safe clinical integration. Many existing AI models generate deterministic point predictions without explicit uncertainty quantification, making it difficult for clinicians to assess the reliability of individual model outputs~\cite{olsson2022estimating}. This concern is particularly relevant for report generation, where clinically plausible language may still contain unsupported or hallucinated content~\cite{wang2025semantic}. Conformal prediction provides a statistical framework for attaching calibrated uncertainty estimates to model outputs, thereby enabling explicit control of predictive risk and the detection of unreliable predictions~\cite{vovk2005algorithmic, angelopoulos2025learn}. Previously, conformal prediction-based methods have found pathology applications in classification and segmentation tasks, but they have remained unexplored for open-ended pathology report generation~\cite{ olsson2022estimating,Zhang2026_TRUECAM,WHSJ21}. Building on this principle, we develop \textbf{CorePath-CRG}, a \textbf{Conformalized} \textbf{Report} \textbf{Generation} extension of \textbf{CorePath} for uncertainty-aware classification and report generation. By issuing outputs only when confidence is sufficient and abstaining otherwise, CorePath-CRG is intended to support safer and more clinically deployable AI-assisted reporting for breast CNB, particularly across heterogeneous real-world settings.

Here, we developed CorePath, a breast-specialized pathology foundation model adapted from a multimodal pathology foundation model through paired breast CNB WSI-report supervision and parameter-efficient Perceiver-only fine-tuning. We further developed CorePath-CRG, a risk-controlled report-generation framework designed to improve the reliability of pathology report generation. We evaluated CorePath across multicenter private cohorts and public breast pathology benchmarks to determine whether it can support breast disease classification at progressively finer diagnostic granularities, ranging from cancer detection and invasion-status categorization to histological subtyping. In addition, we assessed CorePath-CRG for hallucination control and clinical report quality. Collectively, these models provide a clinically oriented framework for breast CNB interpretation, aiming to improve diagnostic accuracy, cross-institutional robustness, and reporting reliability in high-risk clinical scenarios.

\section{Methods}
\label{sec:methodology}
\subsection{Breast Core Needle Biopsy Samples and Cohorts}
A total of 10946 Hematoxylin and Eosin (H\&E)-stained breast CNB slides were retrospectively included. All slides were collected from West China Hospital (WCH), Sichuan Provincial People's Hospital (SPH), the Affiliated Hospital of Southwest Medical University (SWH), West China Tianfu Hospital (WTH), Chengdu Shang Jin Nan Fu Hospital (SJH), and Shenzhen Traditional Chinese Medicine Hospital (SZH), forming training and test cohorts  spanning from June 2019 to September 2024, enabling a comprehensive evaluation of cross-center generalization under distribution shifts~\cite{aubreville2023mitosis, ehteshami2017diagnostic}. The WCH cohort was split into WCH-1 (n=5720 WSIs, June 2019 to April 2024) and WCH-2 (n=496 WSIs, May 2024 to August 2024). The SPH cohort was split into SPH-1 (n=2181 WSIs, January 2023 to April 2024) and SPH-2 (n=612 WSIs, May 2024 to September 2024). The WSI-report pairs of WCH-1 and SPH-1 were used for fine-tuning PRISM. Table~\ref{tab:dataset_summary} summarizes the characteristics of these datasets and details their training, calibration, and testing splits. The overview of the above datasets is shown in Table~\ref{tab:diagnostic_task_summary_SPH-2}-~\ref{tab:diagnostic_task_summary_SZH}. The inclusion criteria were as follows: (1) having undergone a breast biopsy with a definitive pathological diagnosis and (2) each slide should be clear and undamaged. Patients were excluded if the corresponding biopsy histopathological sections were unavailable. This study was approved by the ethics committee of participating hospitals (No. 997 in 2025) and abided by the Declaration of Helsinki before using tissue samples for scientific research purposes only. The requirement to obtain informed consent from the participants was waived by the ethics committee.

\begin{table}[ht]
\centering
\caption{Summary of datasets used for model fine-tuning, conformal calibration, risk control, and zero-shot evaluation.}
\label{tab:dataset_summary}
\renewcommand{\arraystretch}{1.12}
\begin{tabular}{p{2.5cm} c l p{6.5cm}}
\toprule
Datasets & WSIs & Usage & Tasks \& Applications \\
\midrule
\multicolumn{4}{l}{\textit{Model development}} \\
\midrule
SPH-1 \& WCH-1 & 7901 & Fine-tuning 
& CorePath construction \\
SPH-2 & 612 & Calibration; Test 
& CorePath-CRG construction; hierarchical diagnosis\textsuperscript{*} \\
\midrule
\multicolumn{4}{l}{\textit{Independent private validation cohorts}} \\
\midrule
SWH & 1392 & Test 
& Report generation; hierarchical diagnosis\textsuperscript{*} \\
WCH-2 & 496 & Test 
& Report generation; hierarchical diagnosis\textsuperscript{*} \\
WTH & 282 & Test 
& Report generation; hierarchical diagnosis\textsuperscript{*} \\
SJH & 128 & Test 
& Report generation; hierarchical diagnosis\textsuperscript{*} \\
SZH & 135 & Test 
& Cancer detection; invasion assessment \\
\midrule
\multicolumn{4}{l}{\textit{Independent public validation cohorts}} \\
\midrule
BCNB & 1058 & Test 
& Invasive carcinoma subtype categorization \\
BRACS & 547 & Test 
& Lesion stratification; fine-grained classification \\
\bottomrule
\multicolumn{4}{l}{\footnotesize \textsuperscript{*}Hierarchical diagnosis includes cancer detection, invasion assessment, and histological subtyping.} \\
\end{tabular}
\end{table}
\FloatBarrier

\subsection{Whole-Slide Image and Report Preprocessing}
All breast CNB WSIs underwent standardized tissue detection and patch extraction before feature encoding. Background regions and non-informative areas were excluded, and tissue-containing regions were tiled into patches matching the input resolution required by each model. For TITAN, CONCH patch-level features were extracted following the TITAN pipeline~\cite{ding2025multimodal}. For PRISM and CorePath, WSI-level representations were obtained from patch features generated by the Virchow patch encoder~\cite{shaikovski2024prism}.

For report preprocessing, each WSI was paired with its original diagnostic pathology report. Because the reports were written in Chinese and exhibited substantial heterogeneity in wording and formatting across institutions, DeepSeek-R1-Distill-Qwen-32B was first used to extract structured diagnostic information, which was subsequently used to define diagnostic labels for downstream classification tasks~\cite{deepseekai2025deepseekr1incentivizingreasoningcapability}. In parallel, the original Chinese diagnostic information was translated into English using OpenBioLLM-70B,  an open-source biomedical large language model, to provide supervision for WSI-report alignment~\cite{OpenBioLLMs}. When needed, the translated reports were manually reviewed to ensure the preservation of key diagnostic entities. Patient-identifying information was not retained in the processed text. Ultimately, our final datasets comprise paired WSI and pathology reports, with each WSI uniquely associated with one diagnostic report.

\subsection{Development of CorePath}
We developed CorePath by fine-tuning the PRISM foundation model on 7901 breast CNB WSI-report pairs collected from the WCH-1 and SPH-1 cohorts~\cite{shaikovski2024prism}. These data covered routine diagnostic breast pathology cases and were used exclusively for supervised multimodal fine-tuning, with no patient or slide overlap with any downstream evaluation or calibration datasets. This separation enabled subsequent evaluation of both temporal generalization within the fine-tuning institutions and cross-institutional generalization to independent external centers.

In this study, we adopted a lightweight \textbf{Perceiver-only} fine-tuning strategy. All trainable parameters were restricted to the slide encoder, which was implemented as a Perceiver module for aggregating patch-level features into slide-level representations~\cite{jaegle2021perceiver}. The text encoder and all cross-modal components outside the slide encoder were kept frozen throughout training~\cite{kirkpatrick2017overcoming}. By restricting adaptation to the visual encoder, CorePath was encouraged to learn breast-specific histopathological representations while preserving the linguistic embedding space inherited from pretraining~\cite{sung2022vl}. Notably, keeping the language components frozen prevents the model from overfitting to fixed linguistic patterns in the training corpus and exploiting text distributional biases that are irrelevant to the visual content.

Parameter-efficient fine-tuning was implemented by inserting adapters into targeted feed-forward sublayers within the Perceiver architecture~\cite{houlsby2019parameter}.  These targeted modules include the feed-forward network of the final Transformer block, as well as those associated with the cross-attention and aggregation stages. Consequently, parameter updates are strictly confined to this compact subset of layers, while all remaining parameters of PRISM remain frozen. Comprehensive implementation details, including specific layer mappings, hyperparameter configurations, and learning rates, are provided in Supplementary Methods~\ref{app:implementation_details}.

This strictly constrained fine-tuning protocol ensures that adaptation focuses on visual pattern abstraction at the whole-slide level, rather than the memorization of textual templates or reporting conventions. As a result, CorePath was suitable for downstream zero-shot evaluation, uncertainty-aware prediction, and risk-controlled inference across heterogeneous clinical centers~\cite{zhou2022learning}.

\subsection{Zero-Shot Diagnostic Classification Evaluation of CorePath}
To assess the diagnostic generalization and practical adaptability of CorePath, we evaluated the model in a zero-shot setting without task-specific retraining. Zero-shot classification assigns slides to predefined diagnostic categories using text prompts, rather than a task-specific classifier trained on labeled examples for each category set~\cite{radford2021learning}. This strategy is well suited to breast CNB diagnosis, where hierarchical and workflow-dependent label spaces make exhaustive task-specific annotation impractical. For each task, we defined candidate diagnostic prompts and performed subtype prediction without updating model parameters. The prompts used in the zero-shot experiments are listed in Tables~\ref{tab:binary_prompt_class_names}-\ref{tab:bracs_seven_class_prompt_class_names}.

We evaluated PRISM, TITAN, and CorePath on six independent private evaluation cohorts: SPH-2 and WCH-2 served as temporal holdout cohorts from the two institutions used for CorePath fine-tuning, whereas SWH, WTH, SJH, and SZH served as institutionally external cohorts. These cohorts comprised 612, 1392, 496, 282, 128, and 135 cases, respectively. All private evaluation cohorts consisted of breast CNB specimens, consistent with the tissue context of the fine-tuning data, and had no patient or slide overlap with the WCH-1 or SPH-1 fine-tuning cohorts. As summarized in Table~\ref{tab:dataset_summary}, the private cohorts were used to evaluate a clinically progressive hierarchical diagnosis setting. For the internal cohorts, hierarchical diagnosis comprised three diagnostic levels: cancer detection, invasion assessment, and histological subtyping. Cancer detection distinguished cancer from noncancer. Invasion assessment separated noncancer, carcinoma in situ, and invasive breast carcinoma. Histological subtyping distinguished noncancer, carcinoma in situ, invasive ductal carcinoma (also referred to as invasive breast carcinoma of no special type, invasive breast carcinoma not otherwise specified, and invasive breast carcinoma NOS), invasive lobular carcinoma, and other invasive breast carcinoma. These diagnostic levels were designed to reflect the stepwise workflow of breast CNB evaluation, ranging from initial cancer screening to assessment of invasion status and histological subtype determination. This design enabled us to examine whether CorePath improves not only coarse-level cancer detection but also more nuanced diagnostic distinctions that directly inform pathology reporting and downstream clinical management. The SZH cohort was available only for cancer detection and invasion assessment.

Two public benchmarks were also included for external evaluation. The BCNB dataset included 1,058 CNB cases and was used for invasive carcinoma subtype categorization, with invasive ductal carcinoma, invasive lobular carcinoma, and other invasive breast carcinoma as candidate labels~\cite{xu2021predicting}. The BRACS dataset included 547 H\&E-stained mastectomy or CNB specimens and was used for lesion stratification and fine-grained breast lesion classification~\cite{brancati2022bracs}. In the lesion stratification task, BRACS cases were grouped into benign, atypical, and malignant lesions. In the fine-grained classification task, the candidate labels comprised normal tissue, pathological benign lesion, usual ductal hyperplasia, flat epithelial atypia, atypical ductal hyperplasia, ductal carcinoma in situ, and invasive carcinoma. Because the label systems and specimen contexts differed between the private and public datasets, the public benchmarks were used primarily to assess representational transferability and external robustness rather than direct label-level comparability. 

\subsection{Development of CorePath-CRG for Reliable Report Generation}

CorePath-CRG was developed as a conformalized and risk-controlled extension of CorePath for reliable pathology report generation in breast CNB. The framework comprised three operational stages: conformal calibration of diagnostic predictions, report scoring and Learn-Then-Test (LTT)-based threshold calibration, and selective release with hierarchical diagnostic fallback. These stages integrated task-specific conformal filtering, diagnosis-informed assessment using the Self-Confidence Score, a separate ground-truth-aligned Judge Score for calibration-risk definition, LTT-based release-threshold selection, and agent-based output construction~\cite{vovk2005algorithmic,angelopoulos2025learn}. The SPH-2 cohort was used for diagnostic-confidence and report-release calibration, whereas SWH, WCH-2, WTH, and SJH were used for subsequent report-quality evaluation (Table~\ref{tab:dataset_summary}).

For each WSI, CorePath generated diagnostic probabilities through zero-shot inference together with multiple candidate pathology reports. In the first stage, classifier predictions were assessed using task-specific conformal thresholds. For cancer detection, the malignancy probability was compared with lower and upper thresholds. For histological subtyping, the margin between the two highest class probabilities was compared with a calibrated margin threshold. Both tasks yielded a retained diagnostic label or an ``Unknown'' status. Retained labels served as auxiliary diagnostic context, whereas rejected predictions were not treated as supported diagnoses. The task-specific retained or rejected states were also used for hierarchical diagnostic fallback.

In the second stage, an LLM auditor jointly assessed the candidate reports for each case using conformal-filtered subtype and binary cancer status predictions as auxiliary context. The auditor returned a single Self-Confidence Score ranging from 0 to 1 for the candidate-report set. The assessment covered breast pathology relevance, diagnostic completeness, inter-report consistency, alignment with available classifier predictions, and linguistic coherence. No ground-truth reference reports were supplied to this auditor. The resulting score was combined with a diagnostic-confidence component to obtain the fused release score $S_i$. During calibration, a separate ground-truth-aligned Judge Score defined binary risk labels for candidate outputs, and the LTT procedure used these labels and the corresponding release scores to select the release threshold. This threshold was fixed for subsequent inference.

In the third stage, CorePath-CRG implemented three output tiers: automatic narrative release, diagnostic-only fallback, and full rejection. Cases that met the calibrated release threshold were routed to trusted synthesis, in which a synthesizer agent integrated the retained candidate reports and available conformal-filtered diagnostic evidence into a concise diagnostic statement for automatic release. For cases that did not meet the release threshold, the candidate narrative was withheld and diagnostic fallback was applied. A retained histological subtype label was returned preferentially; if the subtype prediction was rejected, a retained binary cancer-status label (cancer or noncancer) was returned instead. Cases were assigned an ``Unknown'' output and deferred for pathologist review only when the narrative failed the release criterion and neither diagnostic prediction was retained.

This design separated reference-based calibration from inference-time assessment and output construction. The Judge Score was used only during calibration. At inference, decisions used the release score, the fixed threshold, and the task-specific diagnostic states and retained labels, without requiring ground-truth reference reports. Final responses followed the selected output tier and available diagnostic evidence. Detailed implementations are provided in Supplementary Methods~\ref{conformal}, \ref{subsec:prism_score_details}, \ref{subsec:judge_score_details}, \ref{riskcontrol}, and \ref{subsec:agent_synthesis_details} for conformal calibration, the Self-Confidence Score, the Judge Score, LTT-based risk control, and synthesis and fallback handling, respectively. The corresponding prompts are shown in Figures~\ref{fig:prism_prompt}, \ref{fig:judge_prompt}, and \ref{fig:agent_prompts}.

\subsection{Report Quality Assessment}
We evaluated generated pathology reports across four independent medical centers, including SWH, WCH-2, WTH, and SJH. To evaluate report generation performance, we compared the reports generated by PRISM, CorePath, and CorePath-CRG. Question-answering foundation models were omitted from the comparison, as their responses are inherently contingent upon the specific questions posed. Report quality was assessed from three complementary perspectives: domain-level hallucination, reference-based textual similarity, and clinical semantic quality. For hallucination analysis, we focused on severe non-breast hallucinations, defined as generated reports containing disease entities or diagnostic descriptions unrelated to breast pathology, thereby reflecting a fundamental misidentification of the pathological domain. Reports generated by PRISM and CorePath were evaluated by a Qwen3 model\footnote[1]{We use the official release checkpoint Qwen3-30B-A3B-Instruct-2507.} to identify any \textit{non-breast hallucinations}~\cite{qwen3technicalreport}. CorePath-CRG was designed for Conformalized Report Generation that incorporates domain-specific guardrails through conformal diagnostic filtering, selective narrative release, and hierarchical diagnostic fallback to decrease hallucination rate. However, hallucination evaluation for CorePath-CRG is excluded to maintain a fair evaluation under identical unconstrained generation conditions.

To evaluate clinical utility and semantic fidelity beyond surface-level lexical overlap, we used a large language model (LLM)-as-a-judge Evaluation Score implemented through the DeepSeek API. For each case, the generated report was compared with the ground-truth reference report under an honesty-prioritized rubric in which correct information was favored over omitted or unknown information, and omitted or unknown information was favored over incorrect or hallucinated content. The Evaluation Score aggregated four clinically critical dimensions with predefined weights: \textit{Clinical Factuality \& Honesty} (40\%), \textit{Clinical Completeness} (30\%), \textit{Logical Consistency} (20\%), and \textit{Professionalism \& Fluency} (10\%). The resulting weighted composite score ranged from 0 to 1, with higher values indicating greater clinical accuracy, completeness, internal consistency, and professional readability. Additional implementation details for the Evaluation Score, including report extraction, scoring workflow, and the full LLM-as-a-judge rubric, are provided in Supplementary Methods~\ref{subsec:eval_score_details} and Figure~\ref{fig:eval_prompt}.

To validate the LLM-based Evaluation Score against pathologist assessment, we conducted an additional blinded, repeated pathologist review of a score-stratified subset of reports. The mean score across three independent review sessions was used as the pathologist reference score. Detailed procedures for report sampling, blinding, and scoring are provided in Supplementary Methods~\ref{subsec:human_validation_details}.

To quantify textual agreement with reference pathology reports, we also computed standard automatic text-similarity metrics, including BLEU, ROUGE, and METEOR (formal definitions provided in Supplementary Methods~\ref{app:metrics}). BLEU measures $n$-gram precision with a brevity penalty, capturing surface-level overlap with reference reports. ROUGE evaluates recall-oriented $n$-gram and subsequence matching, emphasizing content coverage. METEOR extends beyond exact matching by incorporating synonymy and stemming, offering a more semantically aware measure of lexical similarity. All metrics were computed using their standard definitions and default parameter settings. PRISM, CorePath, and CorePath-CRG were evaluated in SWH, WCH-2, WTH, and SJH. 

\subsection{Evaluation and Statistical Analysis}
Diagnostic classification performance was evaluated using accuracy, the weighted F1 score, and the weighted area under the receiver operating characteristic curve (AUC). For multiclass tasks, AUC was computed using a one-vs-rest approach and averaged across classes with class-frequency weighting. Classification metrics were calculated on the complete test sets, with 95\% confidence intervals estimated by nonparametric bootstrapping using 1000 resamples. Report-generation quality was summarized using the hallucination, text-similarity, and Evaluation Score metrics described above. LLM-based Evaluation Scores were compared across the three matched model groups using a one-way repeated-measures analysis of variance with the Geisser-Greenhouse correction. When the overall difference was significant, Tukey's multiple-comparisons test was used for pairwise comparisons. Statistical analyses were performed using GraphPad Prism 9.5 (GraphPad Software, San Diego, CA). A two-sided $P<0.05$ was considered statistically significant.

For pathologist validation of the LLM-based Evaluation Score, intra-rater reliability across the three independent review sessions was quantified using two-way mixed-effects, absolute-agreement intraclass correlation coefficients (ICC). ICC(A,1) represented the reliability of a single-session score, whereas ICC(A,3) represented the reliability of the mean of three session scores. Agreement between the LLM-based Evaluation Score and pathologist reference score was evaluated overall and separately for PRISM, CorePath, and CorePath-CRG. Spearman's $\rho$ assessed rank-order association, whereas a two-way mixed-effects, absolute-agreement, single-measure ICC, denoted ICC(A,1), quantified absolute agreement between the LLM-based Evaluation Score and the corresponding pathologist reference score. ICC estimates were reported with 95\% confidence intervals. Mean absolute difference quantified the average magnitude of scoring error, and Bland-Altman analysis was used to estimate the mean signed bias (LLM-based Evaluation Score minus pathologist reference score) and the 95\% limits of agreement.
\FloatBarrier
\section{Results}

\subsection{Overview of CorePath and CorePath-CRG}
Figure~\ref{fig:pipeline} summarizes the construction of CorePath and its risk-controlled CorePath-CRG extension. Paired breast CNB WSIs and pathology reports from two institutions are used to fine-tune PRISM and derive CorePath. The performance of CorePath was then assessed in independent multicenter cohorts and public datasets, with prespecified tasks spanning cancer detection, invasion assessment, histological subtyping, and fine-grained lesion classification. This design was intended to evaluate whether breast-specific adaptation improved diagnostic performance and generalizability across institutions, specimen sources, and hierarchical diagnostic tasks.

We further developed CorePath-CRG as a three-tier selective report-generation framework. Conformal calibration retained high-confidence diagnostic predictions, while a fused score combining candidate-report Self-Confidence and diagnostic confidence was compared with an LTT-calibrated release threshold. Based on the release decision and the availability of a confident conformal prediction, CorePath-CRG automatically released a synthesized diagnostic statement, returned diagnostic fallback, or output ``Unknown'' and deferred the case for pathologist review.

\begin{figure}[htbp]
    \centering
    \includegraphics[width=\linewidth,height=0.75\textheight,keepaspectratio]{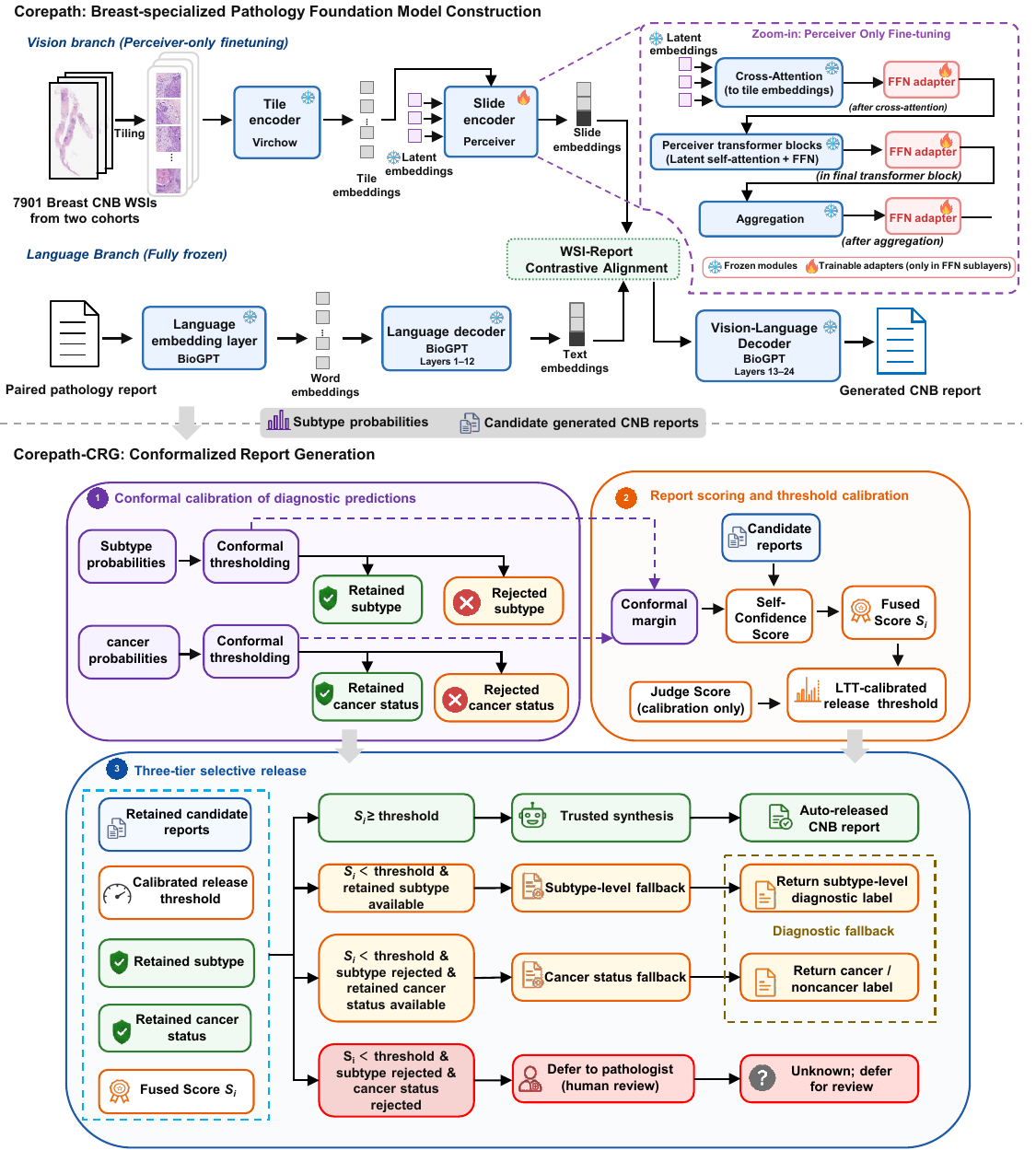}
    \caption{\textbf{Overview of CorePath and CorePath-CRG for breast CNB diagnosis and conformalized report generation.} CorePath was developed by fine-tuning the PRISM pathology foundation model using paired CNB WSIs and pathology reports from two institutions. The resulting model was evaluated using zero-shot diagnostic classification across predefined breast CNB tasks. For report generation, CorePath produced diagnostic probabilities and multiple candidate pathology reports for each case. CorePath-CRG applied conformal filtering to subtype and binary cancer status predictions. Retained diagnostic labels provided auxiliary context for Self-Confidence Score-based assessment of the candidate-report set. The resulting score was combined with a diagnostic-confidence component to obtain the fused release score, while the Judge Score was used only for calibration-risk definition. An LTT-calibrated threshold governed three output tiers: trusted synthesis and automatic report release, diagnostic fallback, or full rejection. When a narrative was not released, fallback returned a retained subtype label or, if unavailable, a retained cancer/noncancer label. Cases were deferred with an ``Unknown'' output only when the narrative failed the release criterion and neither diagnostic label was retained. CNB, core needle biopsy; WSI, whole-slide image; FFN, feed-forward network; LTT, learn-then-test.}
    \label{fig:pipeline}
\end{figure}
\FloatBarrier

\subsection{CorePath Performance Across Breast Pathology Tasks}
CorePath demonstrated consistently strong zero-shot performance in breast tumor subtype prediction across independent external cohorts. Figure~\ref{fig:zeroshot} and Tables~\ref{tab:zeroshot_binary}-\ref{tab:zeroshot_open} summarize the results. Across the private multicenter CNB cohorts, CorePath consistently outperformed the base PRISM model throughout the hierarchical diagnostic workflow defined in Table~\ref{tab:dataset_summary}, from cancer detection to invasion assessment and histological subtyping. For cancer detection, CorePath outperformed PRISM on all three metrics in every cohort and achieved the highest weighted AUC in all six cohorts, with values ranging from 0.9669 to 0.9989 (Figure~\ref{fig:zeroshot}a, Table~\ref{tab:zeroshot_binary}). It also achieved the highest accuracy and weighted F1 score in SPH-2, SWH, WCH-2, and WTH. In SJH and SZH, TITAN achieved higher accuracy and weighted F1 scores, whereas CorePath achieved the highest weighted AUC.

For invasion assessment, CorePath achieved the highest accuracy, weighted F1 score, and weighted AUC in all six private cohorts, with weighted AUCs ranging from 0.9643 to 0.9881 (Figure~\ref{fig:zeroshot}b, Table~\ref{tab:zeroshot_3class}). Its advantage remained pronounced in histological subtyping, a more demanding task requiring discrimination among noncancer, carcinoma in situ, invasive breast carcinoma NOS, invasive lobular carcinoma, and other invasive breast carcinoma (Figure~\ref{fig:zeroshot}c, Table~\ref{tab:zeroshot_5class}). CorePath achieved the highest values for all three metrics in evaluated centers, with weighted AUCs ranging from 0.9526 to 0.9735. These findings indicate that breast-specific multimodal fine-tuning improves generalization for clinically relevant subtype discrimination under cross-center distribution shifts.

The public benchmarks provided complementary evidence of external generalization (Figure~\ref{fig:zeroshot}d, Table~\ref{tab:zeroshot_open}). In the BCNB cohort, an external CNB dataset for invasive carcinoma subtype classification, CorePath achieved the best overall performance, with the highest accuracy of 0.8989, weighted F1 score of 0.8919, and weighted AUC of 0.7780, supporting its generalizability beyond the participating institutions within the same specimen type. In the BRACS cohort, which contains H\&E-stained tissue samples obtained by biopsy or surgical resection and thus represents a broader specimen context, CorePath achieved the highest values for all three metrics in both lesion stratification and fine-grained classification, with weighted AUCs of 0.8178 and 0.8252, respectively. Because the public benchmarks differ from the private CNB cohorts in specimen composition and label definitions, these results are interpreted as evidence of representational transferability and external robustness. Taken together, these observations show that fine-tuning PRISM on paired breast CNB WSI-report data yields consistent gains across cancer detection, invasion assessment, histological subtyping, invasive carcinoma subtype categorization, lesion stratification, and fine-grained classification, supporting the use of CorePath as a more generalizable diagnostic backbone for downstream breast CNB analysis.

\begin{figure}[htbp]
    \centering
    \includegraphics[width=\linewidth]{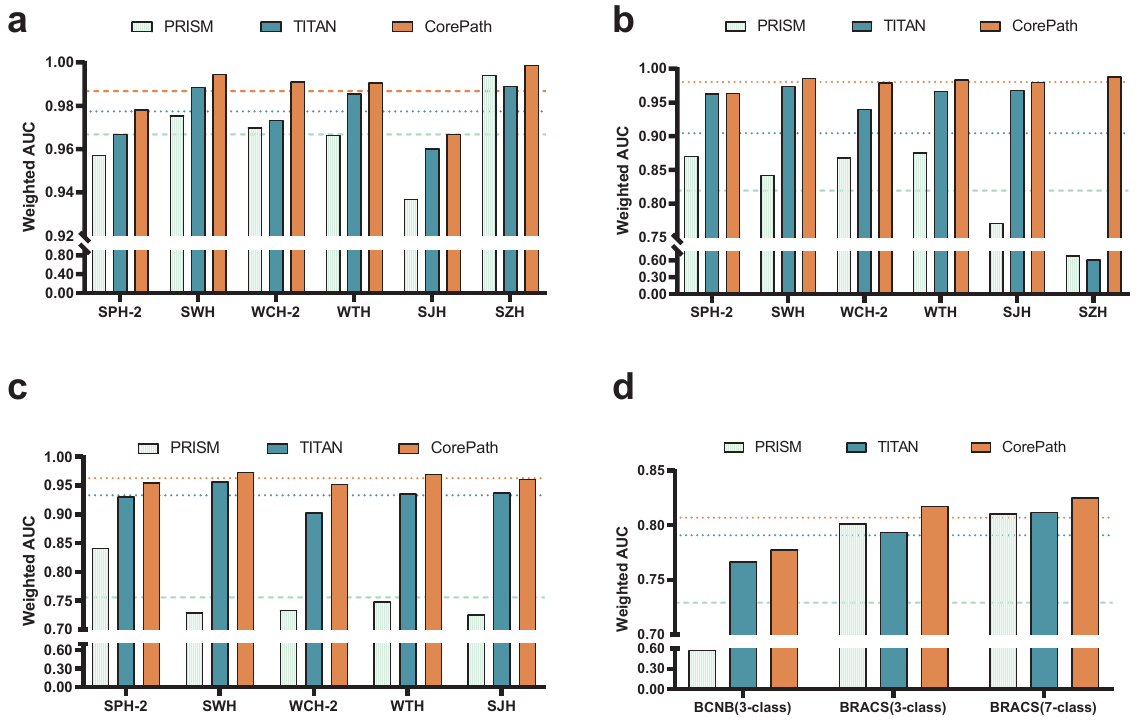}
    \caption{ \textbf{Zero-shot diagnostic classification performance across independent private and public breast pathology cohorts.} (a) Cancer detection on independent private CNB cohorts.
(b) Invasion assessment on independent private CNB cohorts.
(c) Histological subtyping on independent private CNB cohorts.
(d) Public benchmark evaluation on BCNB and BRACS, where the BCNB 3-class task corresponds to invasive carcinoma subtype categorization, and the BRACS 3-class and 7-class tasks correspond to lesion stratification and fine-grained classification, respectively.
Bars show weighted AUCs for PRISM, TITAN, and CorePath; dashed horizontal lines indicate the average weighted AUC of each model within the corresponding panel.
CNB, core needle biopsy; AUC, area under the receiver operating characteristic curve.
}
    
    \label{fig:zeroshot}
\end{figure}
\FloatBarrier

\subsection{Safety and Clinical Quality of Generated Reports}

Beyond diagnostic classification, we evaluated the safety and clinical quality of CorePath and CorePath-CRG for pathology report generation. As shown in Figure~\ref{fig:hallucination_llm}a, CorePath substantially reduced non-breast hallucinations compared with PRISM, decreasing the overall hallucination rate from 30.1\% to 2.8\%. Across individual centers, hallucination rates ranged from 27.6\% to 40.6\% for PRISM and from 2.3\% to 3.9\% for CorePath, supporting the effectiveness of domain-specific fine-tuning in reducing non-breast hallucinations. Building on this improvement, CorePath-CRG provided additional output-level risk control through selective narrative release and diagnostic fallback, with no non-breast hallucinations observed among its released outputs (0\%).

The averaged scores from three repeated pathologist-review sessions showed excellent intra-rater reliability (ICC(A,3) = 0.993; 95\% CI, 0.99-1.00) (Table~\ref{tab:human_reliability}). LLM-based Evaluation Scores showed strong rank correlation and good absolute agreement with the pathologist reference score (Spearman's $\rho=0.911$; ICC(A,1) = 0.843; 95\% CI, 0.74-0.90), supporting their use as a complementary report-quality measure (Table~\ref{tab:llm_human_agreement}, Figure~\ref{fig:bland_altman_agreement}). CorePath achieved higher mean and median Evaluation Scores than PRISM across all four centers, with statistically significant differences in SWH, WCH-2, and WTH but not SJH (Figure~\ref{fig:hallucination_llm}b-e, Table~\ref{tab:llm_eval}). CorePath-CRG further increased mean scores relative to CorePath in all four centers, with significant improvements in WCH-2, WTH, and SJH, whereas the difference in SWH was not significant.

Standard automatic text metrics also supported improved agreement with reference reports. CorePath generally outperformed PRISM, with higher METEOR scores in all four external centers (Table~\ref{tab:hard_metrics}). The largest improvements were observed in WTH, from 0.2124 to 0.2597, and WCH-2, from 0.1944 to 0.2313. CorePath-CRG achieved the highest ROUGE-1, ROUGE-2, and ROUGE-L scores across all four centers, although BLEU and METEOR scores were not consistently higher than those of CorePath. The advantage in all three ROUGE metrics persisted after excluding final ``Unknown'' outputs (Table~\ref{tab:hard_metrics_exclude}). Results for WCH-2 and WTH were unchanged because neither center had final ``Unknown'' outputs.

Hierarchical diagnostic fallback enabled CorePath-CRG to retain diagnostic information when free-text narratives did not meet the release criterion. A histological subtype prediction or, if unavailable, a binary cancer status prediction meeting the confidence criterion could be retained as the final diagnostic output (Table~\ref{tab:pathology_stats}). Although narrative rejection rates ranged from 84.48\% to 89.84\%, diagnostic fallback outputs were retained in 84.48\%-89.06\% of all cases, while full rejection rates remained between 0.00\% and 0.78\% (Table~\ref{tab:risk_stats}). Thus, nearly all narrative-rejected cases still received a diagnostic output rather than undergoing complete deferral. Together, these findings suggest that CorePath improved report quality relative to PRISM, while CorePath-CRG combined additional output-level risk control with further gains in evaluation scores and reference agreement in multiple settings, preserving diagnostic outputs through fallback.

\begin{figure}[htbp]
    \centering
    \includegraphics[width=0.95\linewidth]{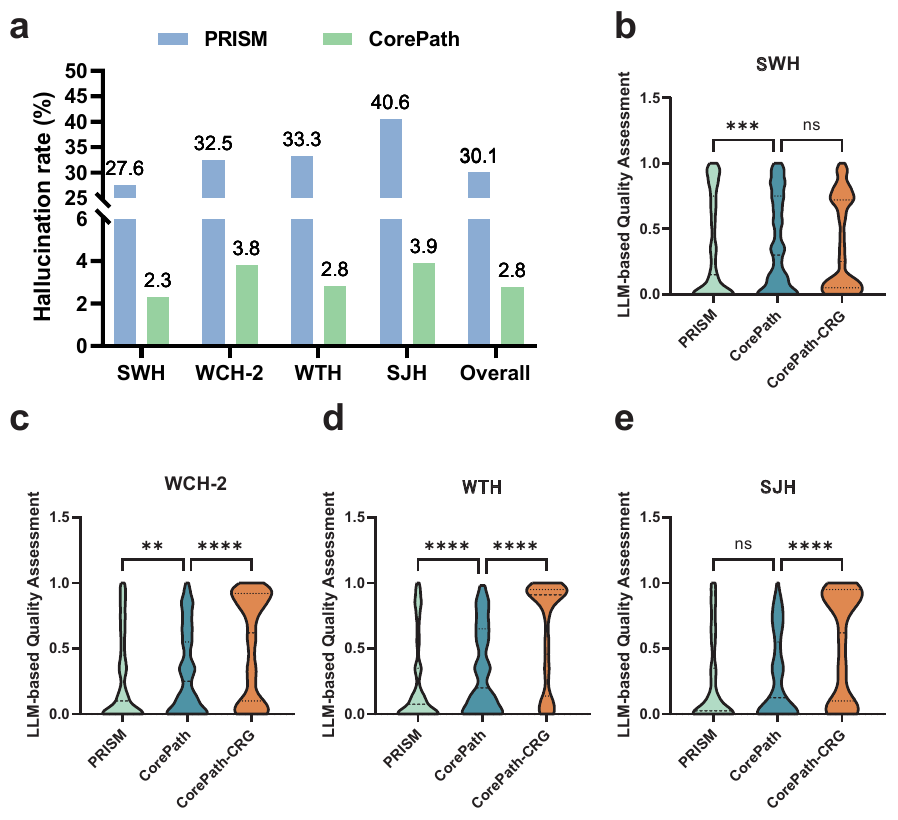}
    \caption{ \textbf{Hallucination reduction and LLM-based Evaluation Score for generated pathology reports.} (a) Hallucination rates of unconstrained PRISM and CorePath outputs. CorePath-CRG is omitted because its 0\% rate reflects risk-controlled selective outputs and is not directly comparable.  
(b-e) Violin plots of LLM-based Evaluation Score distributions for PRISM, CorePath, and CorePath-CRG in SWH, WCH-2, WTH, and SJH. Dashed horizontal lines denote the first quartile, median, and third quartile within each distribution.
ns, not significant; **$P<0.01$; ***$P<0.001$; ****$P<0.0001$.
}
    \label{fig:hallucination_llm}
\end{figure}

\FloatBarrier

\begin{table}[htbp]
\centering
\caption{Quantitative evaluation of PRISM, CorePath, and CorePath-CRG using automatic text-similarity metrics across four datasets. B-1 to B-4 denote BLEU-1 to BLEU-4; R-1, R-2, and R-L denote ROUGE-1, ROUGE-2, and ROUGE-L. Best results are in \textbf{bold}.}
\label{tab:hard_metrics}
\resizebox{\textwidth}{!}{
\renewcommand{\arraystretch}{1.25}
\setlength{\tabcolsep}{4.5pt}
\begin{tabular}{ll cccccccc}
\toprule
\textbf{Dataset} & \textbf{Model} & \textbf{B-1} & \textbf{B-2} & \textbf{B-3} & \textbf{B-4} & \textbf{R-1} & \textbf{R-2} & \textbf{R-L} & \textbf{METEOR} \\
\midrule

\multirow{3}{*}{SWH} 
  & PRISM & 0.1654 & 0.0545 & 0.0333 & 0.0253 & 0.1437 & 0.0225 & 0.1267 & 0.1705 \\
  & CorePath & \textbf{0.1810} & \textbf{0.0567} & \textbf{0.0355} & \textbf{0.0270} & 0.1672 & 0.0199 & 0.1330 & \textbf{0.1728} \\
  & CorePath-CRG & 0.0852 & 0.0438 & 0.0326 & 0.0240 & \textbf{0.1822} & \textbf{0.0767} & \textbf{0.1809} & 0.1203 \\

\midrule

\multirow{3}{*}{WCH-2} 
  & PRISM & 0.1510 & 0.0628 & 0.0405 & 0.0294 & 0.1659 & 0.0406 & 0.1507 & 0.1944 \\
  & CorePath & 0.1702 & 0.0717 & 0.0459 & 0.0321 & 0.2013 & 0.0507 & 0.1801 & \textbf{0.2313} \\
  & CorePath-CRG & \textbf{0.2183} & \textbf{0.0917} & \textbf{0.0713} & \textbf{0.0557} & \textbf{0.3083} & \textbf{0.0509} & \textbf{0.3059} & 0.1893 \\

\midrule

\multirow{3}{*}{WTH} 
  & PRISM & 0.1648 & 0.0576 & 0.0351 & 0.0265 & 0.1449 & 0.0234 & 0.1357 & 0.2124 \\
  & CorePath & 0.1931 & 0.0665 & 0.0427 & 0.0306 & 0.2064 & 0.0341 & 0.1930 & \textbf{0.2597} \\
  & CorePath-CRG & \textbf{0.2893} & \textbf{0.0998} & \textbf{0.0751} & \textbf{0.0643} & \textbf{0.4077} & \textbf{0.0374} & \textbf{0.4075} & 0.2209 \\

\midrule

\multirow{3}{*}{SJH} 
  & PRISM & 0.1483 & 0.0560 & 0.0381 & 0.0293 & 0.1241 & 0.0149 & 0.1139 & 0.1769 \\
  & CorePath & 0.1561 & 0.0489 & 0.0309 & 0.0233 & 0.1506 & 0.0169 & 0.1366 & 0.1931 \\
  & CorePath-CRG & \textbf{0.2542} & \textbf{0.1046} & \textbf{0.0763} & \textbf{0.0616} & \textbf{0.3283} & \textbf{0.0453} & \textbf{0.3283} & \textbf{0.2019} \\

\bottomrule
\end{tabular}
}
\end{table}
\FloatBarrier

\subsection{Visualization and Report Generation Examples}

Figure~\ref{fig:report_examples} presents three representative cases comparing the diagnostic outputs of PRISM, CorePath, and CorePath-CRG against the reference diagnoses. In Case 1, the reference diagnosis was invasive breast carcinoma, not otherwise specified. CorePath generated a breast-specific malignant diagnosis and partially overlapped with the reference, but its description included additional features that were not fully concordant. In contrast, PRISM produced an incorrect diagnosis with mismatched histological terminology. CorePath-CRG generated a concise diagnosis of invasive breast carcinoma NOS and achieved high agreement with the reference. This example suggests that breast-specific fine-tuning improves diagnostic relevance, while the risk-controlled framework can further promote clinically appropriate report release when the model output is sufficiently reliable. In Case 2, the reference diagnosis was invasive carcinoma involving the breast, with consideration of a primary breast origin. Both CorePath and PRISM produced incorrect diagnoses, as reflected by the discordant terms highlighted in red. Rather than releasing a potentially misleading diagnosis, CorePath-CRG returned ``Unknown'', indicating abstention under uncertainty. This case illustrates the safety-oriented behavior of the proposed framework, in which uncertain or low-confidence cases are deferred instead of being automatically reported, prompting manual review by pathologists. In Case 3, the reference diagnosis was fibroadenoma. All three models generated reports consistent with the benign diagnosis. Importantly, CorePath-CRG retained and released the correct benign diagnosis rather than unnecessarily rejecting the case, suggesting that the risk-control mechanism does not simply increase rejection but aims to preserve reliable outputs. Additional examples are illustrated in Figure~\ref{fig:report_examples_supp}.

Overall, these qualitative examples demonstrate the complementary benefits of breast-specific model adaptation and risk-controlled selective reporting. CorePath improves the domain specificity of generated reports compared with the general pathology foundation model, whereas CorePath-CRG further enhances clinical reliability by selectively releasing high-confidence reports and abstaining from uncertain or potentially wrong outputs.

\begin{figure}[htbp]
% \label{fig:example}
    \centering
    \includegraphics[width=.9\linewidth,height=0.88\textheight,keepaspectratio]{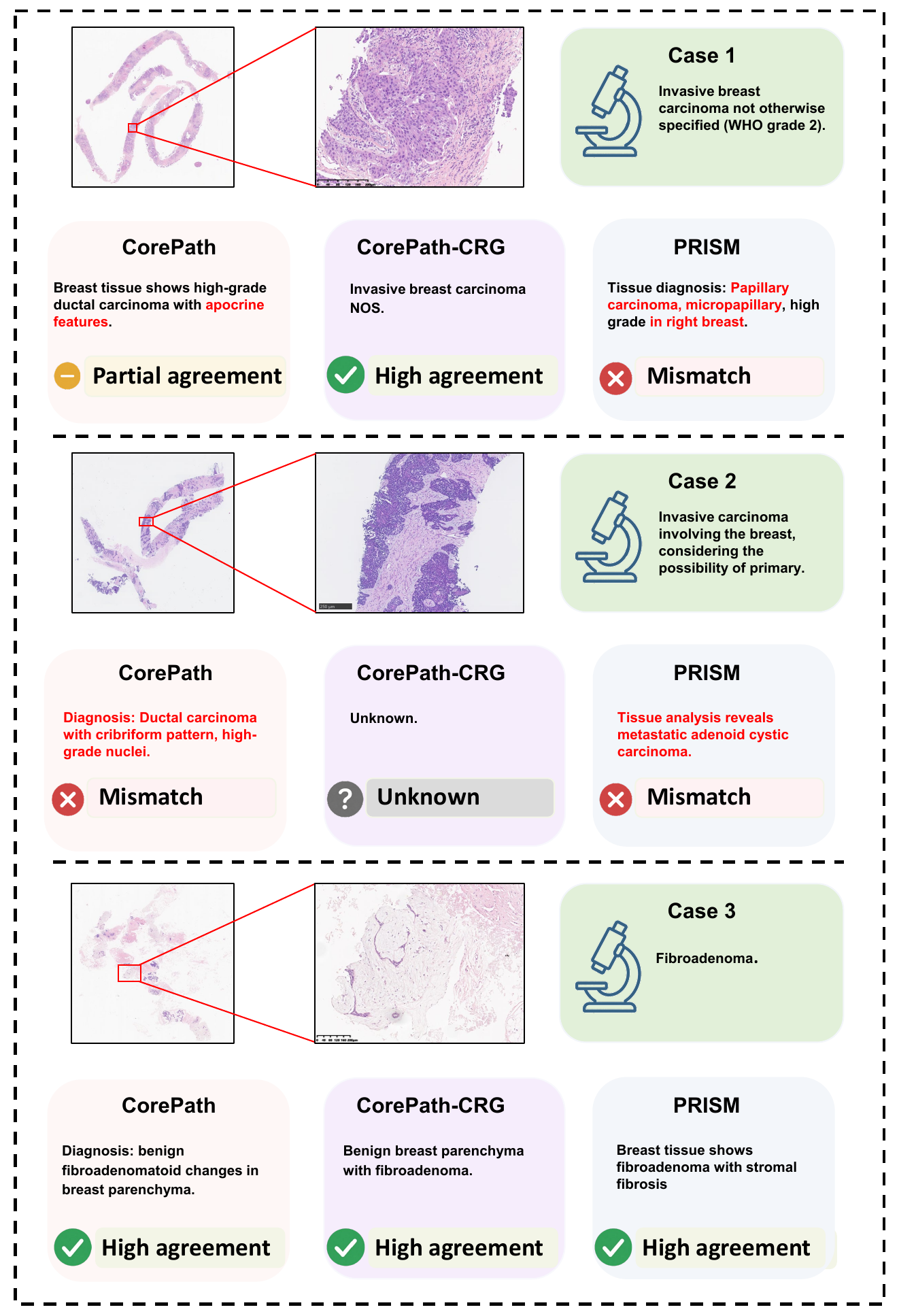}
    \caption{\textbf{Representative report-generation examples across different models.}
For each case, the left panels show the whole-slide image and a magnified region of interest, and the green panel summarizes the reference diagnosis. Model-generated outputs are shown below each case. Red text indicates an incorrect or unsupported diagnosis. Scale bars: from top to bottom, 200$\mu$m, 250$\mu$m, and 200$\mu$m.}
    \label{fig:report_examples}
\end{figure}
\FloatBarrier

\section{Discussion}
In this study, we developed CorePath, a breast-specialized multimodal pathology foundation model for CNB interpretation, and CorePath-CRG, a risk-controlled framework for selective AI-assisted report release. This work addresses the need to adapt broadly pretrained pathology foundation models to clinically specialized diagnostic settings while improving the reliability of generated diagnostic text in high-risk medical workflows. CorePath and CorePath-CRG serve complementary functions, in which CorePath enhances breast-specific diagnostic capability, and CorePath-CRG prevents uncertain or potentially unreliable content from being presented as definitive. This explicit separation between diagnostic inference and output authorization provides the conceptual basis for interpreting our findings and guiding future clinical deployment.

CorePath demonstrated strong diagnostic generalization, suggesting that domain-specific adaptation can further align large pathology foundation models with specialized clinical settings while preserving the flexibility of language-guided zero-shot inference~\cite{ramanathan2025modaltune}. Breast CNB represents a specialized diagnostic setting in which limited tissue sampling, lesion heterogeneity, and subtle morphologic overlap can make subtype discrimination challenging~\cite{schnitt2019problematic}. Although general pathology foundation models have demonstrated strong transferability across histopathology tasks, their representations may not be optimally aligned with the diagnostic hierarchy of breast CNB. By fine-tuning PRISM on paired CNB WSI-report data, CorePath strengthened the association between slide-level morphology and breast-specific diagnostic language, with the most pronounced gains observed in histological subtyping. Importantly, this specialization did not require training separate supervised classifiers for different label systems. Instead, CorePath can be queried using clinically defined diagnostic categories, making it well suited to hierarchical CNB workflows that progress from cancer detection to invasion assessment and histological subtyping. The same principle may also support future prompt-defined label spaces for borderline or clinically challenging entities, such as atypical ductal hyperplasia versus ductal carcinoma in situ, or carcinoma in situ with suspected microinvasion.

Prior work on medical report generation has shown that readable narratives may still fail to capture the clinical accuracy required for clinician-facing workflows, and that generated reports can contain unsupported or hallucinated clinical statements~\cite{ramesh2022improving,asgari2025framework}. Such unsupported diagnostic details may mislead clinicians even when the generated report appears stylistically plausible. In our study, CorePath substantially reduced non-breast hallucinations relative to the base PRISM model, suggesting that domain-specific fine-tuning more closely aligns model generation with breast pathology. However, the residual hallucinations observed with CorePath also show that fine-tuning alone is insufficient for safe report generation. CorePath-CRG adopts a selective-release strategy: narratives meeting the calibrated release criterion are released, whereas rejected narratives are replaced by a retained subtype label or, if no subtype label is retained, a retained binary cancer-status label. Cases are deferred only when the narrative is not released and neither diagnostic prediction is retained. This tiered behavior is consistent with uncertainty-aware medical AI, in which selective prediction and abstention are used to surface uncertainty and defer low-confidence cases rather than forcing an answer~\cite{kompa2021second}. At the same time, the risk-control guarantees provided by CorePath-CRG require careful translation into practice. Conformal prediction provides distribution-free, finite-sample uncertainty guarantees under standard calibration assumptions, whereas LTT calibrates predictive algorithms to achieve explicit finite-sample risk-control guarantees without model refitting~\cite{angelopoulos2023conformal,angelopoulos2025learn}. In pathology workflows, however, calibration must remain relevant to the deployment setting, because external scanners, laboratories, and assessment practices can introduce systematic differences that affect reliability~\cite{olsson2022estimating}. Real-world deployment will therefore require local calibration, continuous monitoring, and periodic recalibration as case mix, scanners, staining protocols, and reporting practices change. The formal guarantee applies to the auto-released report tier, whereas fallback and deferral pathways serve as additional safeguards. Their effects on diagnostic accuracy, turnaround time, workload, and user trust should be evaluated prospectively.

Several limitations should be acknowledged. First, this was a retrospective study, and prospective validation in real clinical workflows is required before clinical deployment. Second, although the evaluation included multiple independent private centers and two public benchmarks, the fine-tuning data were drawn from only two centers. Additional external cohorts are needed to assess robustness across broader variations in tissue processing, staining, scanning platforms, and reporting conventions. Third, the statistical guarantees of the conformal and LTT procedures rely on calibration and test examples being exchangeable, and their validity may be affected by substantial distribution shifts, changes in clinical workflow, or modifications to the underlying model and prompting strategy after calibration. Fourth, our hallucination analysis focused primarily on non-breast hallucinations, which capture severe domain-level errors but do not encompass all possible factual inaccuracies, such as unsupported tumor grade, biomarker status, laterality, or subtle morphologic overinterpretation. However, this information could be reflected indirectly in the LLM-as-a-judge clinical quality evaluation. Fifth, while the LLM-based Evaluation Scores demonstrate strong rank correlation and absolute agreement with pathologist reference scores, the LLM evaluator tended to assign lower scores than the pathologist. Future evaluations could integrate a domain-specific knowledge base to better resolve semantic equivalence and entailment in diagnostic narratives.

Taken together, CorePath demonstrates that paired WSI-report supervision can specialize a multimodal pathology foundation model for breast CNB diagnosis while preserving zero-shot flexibility across clinically relevant diagnostic tasks. CorePath-CRG further integrates breast-specific adaptation with selective narrative release and diagnostic fallback, retaining subtype or binary cancer status outputs when narratives are withheld and deferring cases when neither diagnostic prediction is retained. These findings suggest that domain-specialized and statistically risk-controlled pathology AI may offer a practical path toward more reliable clinical decision support. Before clinical implementation, prospective studies involving pathologist review of AI-assisted diagnoses and reports are required to assess diagnostic accuracy, reporting efficiency, calibration stability, patient safety, pathologist workload, and user trust in real-world clinical workflows.

\section*{Code and Data Availability}
The private institutional WSIs and diagnostic reports used for model adaptation and validation are not publicly available due to patient privacy, institutional data governance, and ethical restrictions. Public benchmark datasets analyzed in this study are publicly available from their original repositories, including BCNB (https://bcnb.grand-challenge.org/) and BRACS (https://www.bracs.icar.cnr.it/download/). The pretrained weights of the baseline models are publicly available from Hugging Face, including PRISM (https://huggingface.co/paige-ai/Prism) and TITAN (https://huggingface.co/MahmoodLab/TITAN). The preprocessing, training, evaluation, statistical analysis, and risk-control code, together with the trained CorePath model weights, will be made publicly available (https://github.com/danninglee/CorePath) upon acceptance of the manuscript to support reproducibility. Qualified academic researchers with approved access to PRISM may reproduce and use the adaptation pipeline under the applicable license terms.

\bibliographystyle{unsrt}
\bibliography{refs}
\newpage

\renewcommand{\thesection}{\Alph{section}} % corrected redefinition of '\thesection'
\makeatletter
\renewcommand\@seccntformat[1]{\csname the#1\endcsname.\hspace{0.5em}}
\makeatother

\vspace{2em}
\begin{center}
  \LARGE \textbf{Supplementary}
\end{center}
\vspace{1.5em}

\setcounter{section}{0}
\section{Supplementary Methods}
\setcounter{table}{0}
\renewcommand{\thetable}{S\arabic{table}}
\renewcommand{\theHtable}{S\arabic{table}}

\subsection{Implementation Details}
\label{app:implementation_details}
\paragraph{Hardware and Training Environment}
All fine-tuning experiments were conducted in a distributed manner using Distributed Data Parallel (DDP) across 8 NVIDIA RTX 3090 GPUs. 

\paragraph{Optimization and Hyperparameters}
The model was optimized for a total of 6 epochs. We utilized a per-GPU batch size of 5, coupled with 16 gradient accumulation steps, resulting in an effective global batch size of 640 ($5 \times 16 \times 8$). The learning rate was initialized at $2 \times 10^{-4}$ with a weight decay of $5 \times 10^{-6}$. The overall training objective comprised multiple loss terms: the weights for both the contrastive loss and the distillation loss were set to 1.0, while the generative loss component was explicitly disabled (weight = 0.0) for this specific fine-tuning phase.

\paragraph{Input Processing Configuration}
For whole-Slide Image (WSI) processing, the maximum number of tiles was capped at 2,560. To introduce stochasticity and mitigate overfitting during training, random tile sampling was enabled. For the text modality, the maximum sequence length was restricted to 768 tokens.

\paragraph{Parameter-Efficient Fine-Tuning via Adapters}
To efficiently adapt the pre-trained PRISM architecture while keeping the computational footprint manageable, we injected lightweight adapter modules exclusively into the Slide Encoder (Perceiver) component. The adapter modules were configured with a bottleneck dimension of 8, a ReLU activation function, a dropout rate of 0.0, and a residual scaling factor of 1.0. Rather than updating the full parameter space, we strategically targeted the feed-forward networks (FFNs) of specific self-attention and cross-attention layers within the Perceiver blocks. The exact targeted modules include the FFNs in the 6th transformer layer of the first block (\texttt{layers.0.tf.5.ff}), as well as the FFNs of the cross-attention layers in the first and second blocks (\texttt{layers.0.xattn.ff} and \texttt{layers.1.xattn.ff}).

\subsection{Conformal Calibration of Diagnostic Predictions}
\label{conformal}

Let $X$ denote the slide-level representation and $Y$ the ground-truth label. Given a classifier outputting predictive probabilities $\hat{p}(y \mid X)$, conformal prediction (CP) constructs a prediction set $C(X)$ that guarantees a finite-sample coverage:
\begin{equation}
  \mathbb{P}(Y \in C(X)) \ge 1 - \alpha,
\end{equation}
where $\alpha \in (0,1)$ is the target error rate. Crucially, standard CP utilizes \emph{all} calibration samples to strictly maintain this theoretical guarantee.

For the binary cancer detection task, the model outputs the malignancy probability $p = \hat{p}(\text{malignant} \mid X)$. We adopt the Least Ambiguous set-valued Classifier~\cite{sadinle2019least}, where the non-conformity score is defined as:
\begin{equation}
  B(X, Y) = 1 - \hat{p}(Y \mid X).
\end{equation}
From a calibration set of size $n$, we compute the finite-sample corrected quantile $\hat{q}$ of the empirical scores at the level $\lceil(n+1)(1-\alpha)\rceil/n$. This yields two decision thresholds:
\begin{equation}
  t_{\mathrm{upper}} = 1 - \hat{q}, \quad t_{\mathrm{lower}} = \hat{q}.
\end{equation}
At inference, a sample is classified as reliably malignant if $p \ge t_{\mathrm{upper}}$, reliably benign if $p \le t_{\mathrm{lower}}$, and deferred as ambiguous if $p$ falls between the thresholds.

For multi-class histological subtyping, we employ a margin-based CP approach~\cite{vovk2005algorithmic}. The non-conformity score is defined as the negative margin of the true class relative to the strongest competing class:
\begin{equation}
  M(X, Y) = \max_{y' \neq Y} \hat{p}(y' \mid X) - \hat{p}(Y \mid X).
\end{equation}
After computing the calibrated quantile $\hat{q}$ from all calibration samples, we set the reliability threshold to the absolute value of the quantile to ensure a prediction set of size exactly one:
\begin{equation}
  t_{\mathrm{margin}} = |\hat{q}|.
\end{equation}
At inference, let $\Delta = p_{(1)} - p_{(2)}$ be the margin between the top two predicted probabilities. The model outputs a definitive subtype only if $\Delta > t_{\mathrm{margin}}$; otherwise, it abstains. This mechanism rigorously prevents overconfident misclassifications in both diagnostic tasks.

\subsection{Self-Confidence Score for Candidate-Report Reliability Assessment}
\label{subsec:prism_score_details}

\paragraph{Overview and Workflow.}
The Self-Confidence Score is a continuous score ranging from 0 to 1 that summarizes the overall quality of a case's candidate-report set. The scoring prompt takes five candidate reports together with conformal-filtered subtype and binary cancer status predictions as inputs. An LLM auditor jointly assesses these inputs and returns structured assessment notes and a single score for the report set.

\paragraph{Uncertainty-Aware Auxiliary Context via Conformal Prediction.}
The upstream subtype and binary cancer status predictions are supplied to the LLM auditor as auxiliary diagnostic context. Before scoring, task-specific conformal filtering determines which classifier labels are retained; rejected predictions are represented as \texttt{Unknown}.

While the reference report serves as the primary standard for evaluation, the upstream classifier's predictions are provided to the LLM auditor as auxiliary context to aid the auditing process. To prevent the auditor from unfairly penalizing AI-generated reports that contradict \textit{uncertain} classifier predictions, we utilize Conformal Prediction to dynamically filter out low-confidence auxiliary labels.
\begin{itemize}
    \item \textbf{Binary Malignancy Status:} Given the classifier's confidence score $s_{bi}$, the auxiliary prediction is deemed reliable only if it falls outside the ambiguity region (i.e., $s_\mathrm{bi} > \tau_\mathrm{upper}$ or $s_\mathrm{bi} < \tau_\mathrm{lower}$). Otherwise, the auxiliary status is masked as \texttt{Unknown}.
    \item \textbf{Multi-Class Subtype Status:} Given the classifier's probability distribution over subtypes, we compute the prediction margin $\Delta = p_\mathrm{max} - p_\mathrm{second\_max}$. If $\Delta > \tau_\mathrm{margin}$, the auxiliary subtype is considered reliable; otherwise, it is masked as \texttt{Unknown}.
\end{itemize}
By marking uncertain classifier predictions as \texttt{Unknown}, the scoring prompt explicitly instructs the LLM auditor to bypass the classifier-alignment constraint for these specific cases, ensuring that the evaluation relies solely on the reliable reference and clinical logic when upstream predictions are noisy.

\paragraph{LLM Auditor and Prompt Design.}
We employ the DeepSeek API (\texttt{deepseek-chat}) with a temperature of 0.0 to guarantee deterministic and reproducible scoring. The LLM is instructed to act as a breast pathology report quality auditor, evaluating the candidate reports across five critical dimensions: (1) Breast relevance, (2) Diagnostic completeness, (3) Inter-report consistency, (4) Information alignment (conditionally applied based on Conformal Prediction), and (5) Linguistic coherence. The exact prompts are detailed in Figure~\ref{fig:prism_prompt}.

\setcounter{figure}{0}
\renewcommand{\thefigure}{S\arabic{figure}}
\renewcommand{\theHfigure}{S\arabic{figure}}
\begin{figure}[htbp]
\centering
\begin{tcolorbox}[
    colback=gray!5,
    colframe=gray!60,
    boxrule=0.5pt,
    arc=0pt,
    left=6pt, right=6pt, top=6pt, bottom=6pt,
    title=\textbf{Prompt Template for Self-Confidence Score (Self-Confidence Scoring Agent)},
    fonttitle=\normalsize\bfseries
]
\begin{lstlisting}[
    breaklines=true,
    basicstyle=\small\ttfamily,
    columns=fullflexible,
    frame=none,
    backgroundcolor=\color{gray!5}
]
### SYSTEM PROMPT ###
You are a breast pathology report quality auditor.
Given 5 AI-generated pathology reports and a reference classifier output, score the overall report quality on a continuous scale from 0.00 to 1.00.

Focus on these aspects when forming your score:
1. Breast relevance -- do the reports clearly discuss breast tissue or breast disease?
2. Diagnostic completeness -- do they contain a meaningful clinical conclusion?
3. Consistency -- do the reports agree with each other on malignancy and subtype?
4. Information alignment -- do the generated reports align with the reference classifier?
5. Linguistic coherence -- are the reports written as coherent medical prose?

Scoring guidance (use the full continuous range, avoid rounding to 0.1 steps):
  - Reports that are relevant, complete, consistent, and reference classifier output-aligned -> near 1.00
  - Partial relevance or mild inconsistency -> 0.50-0.79
  - Mostly irrelevant, incoherent, or contradicting reference classifier output -> below 0.40
  - Completely off-topic or nonsensical -> near 0.00

Output format (use these exact tags, nothing else outside them):
<analysis>
Breast relevance: [brief note]
Diagnostic completeness: [brief note]
Consistency: [brief note]
Information alignment: [brief note or "reference classifier unknown -- skipped"]
Coherence: [brief note]
</analysis>
Self-Confidence_Score: X.XX

### USER PROMPT ###
Reference Classifier: Subtype={mu_status} | Malignancy={bi_status} 
               (low-confidence -- treat as unknown) [Conditionally appended]

Reports:
R1: {report_1}
R2: {report_2}
R3: {report_3}
R4: {report_4}
R5: {report_5}
\end{lstlisting}
\end{tcolorbox}
\caption{System and User prompts used for calculating the Self-Confidence Score via the LLM auditor.}
\label{fig:prism_prompt}
\end{figure}

\paragraph{Response Parsing.}
The LLM's raw output is parsed using regular expressions to extract the textual analysis enclosed within the \texttt{<analysis>} tags and the numerical score following the \texttt{Self-Confidence\_Score:} prefix. Fallback patterns are implemented to capture standard floating-point formats (e.g., \texttt{0.xx} or \texttt{1.00}) to ensure robust extraction even if the LLM slightly deviates from the strict formatting constraints.

% \subsection{Judge Score: Methodology and Prompt Design}
\subsection{Judge Score for Ground Truth-Aligned Clinical Quality Assessment}
\label{subsec:judge_score_details}

\paragraph{Overview and Workflow.}
The Judge Score is designed to evaluate the factual accuracy and diagnostic alignment of the AI-generated reports against expert-annotated references. While the Self-Confidence Score assesses holistic quality and internal consistency using reference classifier predictions, the Judge Score directly compares the generated text with human reference translations to measure clinical correctness. 

\paragraph{Ground Truth Alignment and Report Filtering.}
For each case, the pipeline retrieves multiple equivalent reference translations (typically 5 versions) of the ground truth (GT) diagnosis from an external clinical database, establishing a high-confidence diagnostic consensus. To ensure the LLM auditor focuses only on meaningful outputs and is not distracted by degenerate generations, we apply a heuristic filtering mechanism to the AI-generated candidate reports. A report is retained for evaluation only if it exceeds a minimum character length, contains appropriate sentence-level punctuation, and includes core pathological terminology from a predefined lexicon (e.g., \textit{malignant, benign, carcinoma, invasive, fibroadenoma}). 

\paragraph{LLM Auditor and Prompt Design.}
Similar to the Self-Confidence Score, we employ the DeepSeek API with a temperature of 0.0 to ensure deterministic evaluation. The LLM is instructed to act as a Ground Truth Judge, focusing on three primary dimensions: (1) GT Alignment (correct identification of core malignancy and histological subtype), (2) Inter-Report Consistency, and (3) Linguistic Coherence. The exact prompts are detailed in Figure~\ref{fig:judge_prompt}.

\begin{figure}[htbp]
\centering
\begin{tcolorbox}[
    colback=gray!5,
    colframe=gray!60,
    boxrule=0.5pt,
    arc=0pt,
    left=6pt, right=6pt, top=6pt, bottom=6pt,
    title=\textbf{Prompt Template for Judge Score},
    fonttitle=\normalsize\bfseries
]
\begin{lstlisting}[
    breaklines=true,
    basicstyle=\small\ttfamily,
    columns=fullflexible,
    frame=none,
    backgroundcolor=\color{gray!5}
]
### SYSTEM PROMPT ###
You are a breast pathology report quality auditor (Ground Truth Judge).
Given the Ground Truth (GT) reference translations (multiple equivalent versions) and a set of AI-generated pathology reports, score the factual accuracy and alignment of the AI reports against the GT on a continuous scale from 0.00 to 1.00.

Note: The GT section contains 5 equivalent translation versions of the same diagnosis. Focus on the core diagnostic consensus (e.g., Malignancy, Histological Subtype) present in these translations.

Focus on these aspects when forming your score:
1. GT Alignment -- do the AI reports correctly identify the core malignancy and histological subtype stated in the GT translations?
2. Inter-Report Consistency -- do the AI reports agree with each other on the core diagnosis?
3. Linguistic Coherence -- are the AI reports written as coherent, professional medical prose?

Scoring guidance (use the full continuous range, avoid rounding to 0.1 steps):
  - Perfect GT alignment, high consistency, and coherent -> near 1.00
  - Minor GT mismatches (e.g., subtype ambiguity) or mild inconsistency -> 0.60-0.89
  - Major GT contradictions (e.g., benign vs malignant) or high variance -> 0.20-0.59
  - Completely wrong diagnosis, nonsensical, or no valid reports -> near 0.00

Output format (use these exact tags, nothing else outside them):
<analysis>
GT Core Entities: [Malignancy, Subtype extracted from translations]
AI Consensus: [Malignancy, Subtype]
Alignment: [brief note]
Consistency: [brief note]
Coherence: [brief note]
</analysis>
JUDGE_Score: X.XX

### USER PROMPT ###
GT REFERENCE TRANSLATIONS [Confidence: {gt_confidence}]:
{gt_text}

AI REPORTS (N={n_valid_reports}):
R1: {report_1}
R2: {report_2}
...
\end{lstlisting}
\end{tcolorbox}
\caption{System and User prompts used for calculating the Judge Score via the LLM auditor.}
\label{fig:judge_prompt}
\end{figure}

\paragraph{Response Parsing.}
The parsing logic mirrors that of the Self-Confidence Score, utilizing regular expressions to extract the structured analysis and the continuous \texttt{JUDGE\_Score}. Cases where the GT is unavailable in the external database are explicitly flagged with a low-confidence indicator in the user prompt, instructing the LLM to adjust its alignment expectations accordingly and preventing undue penalization of the AI models.

\subsection{Learn-Then-Test Calibration for Risk-Controlled Report Release}
\label{riskcontrol}

This section outlines the Learn-Then-Test (LTT) framework used to calibrate a release threshold for automated pathology reports, ensuring that the clinical risk of auto-released reports is strictly controlled.

\subsubsection{Notation and Score Construction}
Let $\Xcal$, $\Ycal$, and $\hat{\Ycal}$ denote the spaces of whole slide images, ground-truth reports, and AI-generated candidates, respectively. Let $\Dcal = \{z_i\}_{i=1}^n$ denote the calibration set, where $z_i = (x_i, y_i, \hat{y}_i) \in \Xcal \times \Ycal \times \hat{\Ycal}$. 

We define a binary risk indicator $R: \Xcal \times \Ycal \times \hat{\Ycal} \to \{0,1\}$ based on a reference-aligned judge scoring function $J(z)$:
\begin{equation}
  R(z) = \mathbb{I}\!\left(J(z) < \tau_{\mathrm{judge}}\right),
\end{equation}
where $\tau_{\mathrm{judge}}$ is a pre-specified clinical quality threshold, and $R(z)=1$ indicates a high-risk (clinically unacceptable) report.

To stratify samples by reliability, we construct a continuous fused score $S_i$ for each $z_i$ using the upstream conformal classifier's outputs: the self-confidence score $P_i = \max_{c} \hat{p}_i^{(c)}$ and the predictive margin $M_i = \hat{p}_i^{(c_1)} - \hat{p}_i^{(c_2)}$ (the difference between the top-1 and top-2 predicted probabilities). The fused score is defined as:
\begin{equation}
  S_i = P_i + \lambda M_i + \epsilon_i,
\end{equation}
where $\lambda > 0$ is a weighting hyperparameter, and $\epsilon_i \sim \mathcal{U}(0, \eta)$ (with a negligible $\eta$, e.g., $10^{-6}$) is added solely to ensure a strict total ordering by breaking ties.

\subsubsection{Learn-Then-Test Risk Control}
We sort the calibration samples in descending order of their fused scores: $S_{(1)} \geq S_{(2)} \geq \cdots \geq S_{(n)}$, with corresponding risk indicators $R_{(1)}, \ldots, R_{(n)}$. 
For any $k \in \{1,\ldots,n\}$, we define the \emph{acceptance set} $\Acal_k$ as the subset of the $k$ samples with the highest scores. The number of high-risk samples in $\Acal_k$ is denoted as $V_k = \sum_{j=1}^k R_{(j)}$.

Given a target risk level $\alpha \in (0,1)$ and a statistical confidence parameter $\delta \in (0,1)$, our objective is to find the largest acceptance set (maximizing auto-release coverage) such that the true risk is bounded by $\alpha$ with probability at least $1-\delta$. 

Assuming the calibration samples are exchangeable, we employ the Clopper-Pearson exact confidence bound to compute the $(1-\delta)$-confidence upper bound on the true risk for $\Acal_k$:
\begin{equation}
  U(k, V_k, \delta) = \mathrm{Beta}^{-1}\!\left(1-\delta;\; V_k + 1,\; k - V_k\right),
\end{equation}
where $\mathrm{Beta}^{-1}(\cdot; a, b)$ is the quantile function of the Beta distribution. This non-asymptotic bound is strictly valid for any sample size $k \ge 1$ without requiring parametric assumptions.

The optimal number of samples to auto-release, $k^*$, is determined by finding the maximum $k$ that satisfies the risk constraint:
\begin{equation}
  k^* = \max\!\left\{ k \in \{1,\ldots,n\} : U(k, V_k, \delta) \leq \alpha \right\}.
\end{equation}
If no such $k$ exists, $k^* = 0$ and no reports are auto-released. The final release threshold for inference is set to $\tau_{\mathrm{release}} = S_{(k^*)}$. Any new report with a fused score $S(z) \geq \tau_{\mathrm{release}}$ is auto-released.

\subsubsection{Summary of Guarantees and Practical Considerations}
Under the exchangeability assumption, the LTT framework provides a finite-sample, distribution-free guarantee: the true clinical risk of the auto-released reports satisfies $\mathbb{E}[R(z) \mid S(z) \geq \tau_{\mathrm{release}}] \leq \alpha$ with probability at least $1-\delta$. 

In practice, $\alpha$ and $\delta$ jointly control the trade-off between clinical safety and automation efficiency. A smaller $\alpha$ enforces stricter risk control, while a smaller $\delta$ demands higher statistical confidence; both result in a more conservative threshold $\tau_{\mathrm{release}}$ and lower auto-release coverage. While the exchangeability assumption holds for data collected under consistent protocols, multi-site deployments with significant covariate shift may require weighted extensions of this framework.

\subsection{Agent-Based Report Finalization and Fallback Workflow}
\label{subsec:agent_synthesis_details}

\paragraph{Overview and Tiered Workflow.}
While the Learn-Then-Test (LTT) framework provides a statistically rigorous threshold $\hat{\lambda}_{thresh}$ for risk control, the actual deployment requires an intelligent mechanism to process the AI outputs based on this risk assessment. We design a three-tier Agent Synthesis and Feedback Mechanism that dynamically routes the generated content. Instead of blindly releasing all LLM-generated text, the system employs specialized LLM agents to either synthesize a comprehensive diagnosis, gracefully degrade to a safer output, or intercept the case for human review, ensuring that the feedback provided to the clinician is always commensurate with the system's statistical confidence.

\paragraph{Scoring Agent and Fused Risk Metric.}
Before synthesis, a \textit{Scoring Agent} acts as an internal quality auditor. It evaluates the ensemble of candidate reports (e.g., 5 generated variations) alongside the upstream classifier's predictions. The agent outputs a continuous Self-Confidence Score ($P_i \in [0.00, 1.00]$) assessing breast relevance, diagnostic completeness, consistency, information alignment, and linguistic coherence (Supplementary Methods~\ref{subsec:prism_score_details}). This score is then fused with the conformal predictive margin ($M_i$) to compute the final risk metric $S_i = P_i + \lambda M_i + \epsilon_i$. This continuousization resolves the discrete clustering of LLM scores and enables precise thresholding.

\paragraph{Synthesizer Agent and Tiered Feedback Strategy.}
Based on the fused release score $S_i$, the fixed threshold $\hat{\lambda}_{thresh}$, and the task-specific diagnostic states, the system selected one of three output tiers:
\begin{enumerate}
    \item \textbf{Trusted Synthesis ($S_i \geq \hat{\lambda}_{thresh}$):} The candidate-report set met the release criterion. A \textit{Synthesizer Agent} integrated the retained candidate reports and available conformal-filtered diagnostic evidence into a concise diagnostic statement for automatic release.
    \item \textbf{Hierarchical Diagnostic Fallback ($S_i < \hat{\lambda}_{thresh}$ with a retained diagnostic prediction):} The candidate narrative was withheld. A retained subtype label was returned preferentially. If the subtype prediction was rejected, a retained binary cancer status label (cancer or noncancer) was returned instead. Both branches returned only a diagnostic label.
    \item \textbf{Full Rejection ($S_i < \hat{\lambda}_{thresh}$ with neither diagnostic prediction retained):} When both the subtype and binary cancer status predictions were rejected, the final output was ``Unknown'' and the case was deferred for pathologist review.
\end{enumerate}

\paragraph{LLM Agents and Prompt Design.}
Both the Self-Confidence Scoring and Synthesizer agents are powered by the DeepSeek API (temperature = 0.0) to ensure deterministic and rigorous clinical reasoning. System prompts for the Self-Confidence Scoring Agent used to derive the continuous Self-Confidence Score are detailed in Figure~\ref{fig:prism_prompt}. The exact system prompts governing synthesizer agents are detailed in Figure~\ref{fig:agent_prompts}.

\begin{figure}[htbp]
\centering
\begin{tcolorbox}[
    colback=gray!5,
    colframe=gray!60,
    boxrule=0.5pt,
    arc=0pt,
    left=6pt, right=6pt, top=6pt, bottom=6pt,
    title=\textbf{Prompt Templates for Synthesis Agent},
    fonttitle=\normalsize\bfseries
]
\begin{lstlisting}[
    breaklines=true,
    basicstyle=\small\ttfamily,
    columns=fullflexible,
    frame=none,
    backgroundcolor=\color{gray!5}
]
### SYNTHESIZER SYSTEM PROMPT ###
You are an expert breast pathologist. Your task is to synthesize a set of AI-generated pathology reports and a reference classifier subtype into a single, concise, and professional final diagnosis string.

Rules:
- The output must be a single diagnostic statement.
- Do NOT include any explanations, greetings, or extra text. Output ONLY the diagnosis string.
- Follow the style and brevity of these exact examples:
  1. Invasive carcinoma, favor invasive lobular carcinoma.
  2. Fibroadenoma with focal calcification.
  3. The breast tissue shows extensive acute and chronic inflammatory cell infiltration.
\end{lstlisting}
\end{tcolorbox}
\caption{Prompt for the Synthesizer Agent used to generate the final auto-released clinical diagnosis under the trusted pathway.}
% System prompts for the Scoring Agent (top) used to derive the continuous Proximity Score, and 
\label{fig:agent_prompts}
\end{figure}

\paragraph{Response Parsing and Execution Guarantee.}
The system employs strict regex-based parsing to extract the \texttt{Self-Confidence\_Score} from the \texttt{<analysis>} block, ensuring that the continuousization pipeline is robust against LLM formatting variations. By decoupling the \textit{evaluation} (Self-Confidence Scoring Agent) from the \textit{final output generation} (Synthesizer Agent), and gating the latter behind a statistically validated conformal threshold, this mechanism guarantees that the AI copilot never presents a highly detailed but potentially
hallucinated narrative to the clinician when the underlying statistical confidence is low.

\subsection{Evaluation Score Methodology and Prompt Design}
\label{subsec:eval_score_details}

\paragraph{Overview and Workflow.}
While the Self-Confidence score and Judge score evaluate ensemble consistency and translation alignment, the Evaluation Score is designed to assess the clinical utility and safety of a \textit{single} generated pathology report as it would be presented to a clinician in a real-world deployment. For each case, the pipeline retrieves a single reference report and compares it against one AI-generated report. To accommodate different generation paradigms, we employ a dynamic report extraction strategy: for PRISM and CorePath that generate multiple candidate reports (delimited by \texttt{|||}), one report is randomly sampled for evaluation; for the CorePath-CRG model that produces a single deterministic output, the entire text is evaluated directly.

\paragraph{Honesty Priority and Evaluation Dimensions.}
A core principle of this evaluation is the \textbf{Honesty Priority}, which strictly enforces the clinical safety hierarchy: \textit{Correct Information > Unknown/Omitted > Incorrect/Hallucinated Information}. Fabricating medical facts or making critical diagnostic errors (e.g., confusing benign and malignant lesions) is heavily penalized, whereas omitting uncertain information is strongly preferred over guessing. The LLM auditor evaluates the report across four weighted dimensions:
\begin{enumerate}
    \item \textbf{Clinical Factuality \& Honesty (40\%):} Accuracy of medical entities and strict adherence to the honesty priority.
    \item \textbf{Clinical Completeness (30\%):} Coverage of essential diagnostic information present in the ground truth, including malignancy status, histological subtype, grade when available, and relevant morphologic findings.
    \item \textbf{Logical Consistency (20\%):} Internal coherence, ensuring the microscopic description aligns with the final diagnosis.
    \item \textbf{Professionalism \& Fluency (10\%):} Use of standard breast pathology terminology and coherent medical prose.
    
\end{enumerate}

\paragraph{LLM Auditor and Prompt Design.}
We utilize the DeepSeek API with a temperature of 0.0 and an increased maximum token limit (1024 tokens) to allow for more detailed analytical reasoning. The exact prompts, including the scoring guidance and dimension weights, are detailed in Figure~\ref{fig:eval_prompt}.

\begin{figure}[htbp]
\centering
\begin{tcolorbox}[
    colback=gray!5,
    colframe=gray!60,
    boxrule=0.5pt,
    arc=0pt,
    left=6pt, right=6pt, top=6pt, bottom=6pt,
    title=\textbf{Prompt Template for Evaluation Score},
    fonttitle=\normalsize\bfseries
]
\begin{lstlisting}[
    breaklines=true,
    basicstyle=\small\ttfamily,
    columns=fullflexible,
    frame=none,
    backgroundcolor=\color{gray!5} % 与外部 tcolorbox 背景完美融合
]
### SYSTEM PROMPT###
You are an expert pathology report quality auditor.
Given the Ground Truth (GT) reference report and an AI-generated pathology report,
evaluate the AI report on a continuous scale from 0.00 to 1.00.

[Core Principle: Honesty & Factuality]
Priority: Correct Information > Unknown/Omitted > Incorrect/Hallucinated Information.
Fabricating medical facts, hallucinating non-existent lesions, or making critical diagnostic errors (e.g., benign vs. malignant) must be heavily penalized. Omitting uncertain information is strictly preferred over guessing incorrectly.

[Evaluation Dimensions & Weights]
1. Clinical Factuality & Honesty (40%): Accuracy of stated diagnostic information. Strictly apply the honesty priority.
2. Diagnostic Coverage (30%): Coverage of essential diagnostic information present in the ground truth, including malignancy status, histological subtype, grade when available, and relevant morphologic findings. Do not penalize omission of information that is not present in the ground truth or not assessable on core needle biopsy.
3. Logical Consistency (20%): Internal coherence (e.g., microscopic description aligns with final diagnosis).
4. Professionalism & Fluency (10%): Standard breast pathology terminology and coherent medical prose.

[Scoring Guidance]
- Near 1.00: Accurate, honest, complete, and professional.
- 0.60-0.89: Minor omissions or slight terminology flaws, but NO critical factual errors.
- 0.20-0.59: Minor factual contradictions. 
- Near 0.00: Severe hallucinations, critical diagnostic errors, or nonsensical.

Output format (use these exact tags, nothing else outside them):
<analysis>
Factuality & Honesty: [brief note]
Completeness: [brief note]
Consistency: [brief note]
Professionalism: [brief note]
</analysis>
EVAL_Score: X.XX

### USER PROMPT ###
GT REFERENCE REPORT:
{gt_text}

AI GENERATED REPORT:
{ai_report}
\end{lstlisting}
\end{tcolorbox}
\caption{System and User prompts used for calculating the Evaluation Score, emphasizing the Honesty Priority and weighted clinical dimensions.}
\label{fig:eval_prompt}
\end{figure}

\paragraph{Response Parsing and Clinical Significance.}
The response parsing follows the established regex-based extraction protocol to isolate the \texttt{<analysis>} block and the final \texttt{EVAL\_Score}. By strictly enforcing the honesty priority through the prompt design, this metric provides a highly reliable proxy for clinical safety that can be compared with expert assessment. Crucially, it ensures that models like the Risk-Control (CorePath-CRG) variant are appropriately rewarded for their factual reliability, rather than being penalized by traditional lexical overlap metrics (e.g., BLEU/ROUGE) for missing uncertain information.

\subsection{Pathologist Validation of the LLM-Based Evaluation Score}
\label{subsec:human_validation_details}

To assess the alignment between the Evaluation Score assigned by the DeepSeek API and pathologist judgment, we conducted a blinded and repeated pathologist review using a subset of generated reports selected through score stratification. Twelve cases were selected from each of the four external centers (SWH, WCH-2, WTH, and SJH), yielding a total of 48 cases. Within each center, cases were sampled from four predefined intervals of the CorePath-CRG Evaluation Score, with three cases selected from each interval. For each selected case, the reports generated by PRISM, CorePath, and CorePath-CRG were all included in the evaluation. Consequently, the same 48 cases were evaluated across all three models, resulting in 144 reports for evaluation.

All reports were anonymized and presented in a fully randomized order. Information regarding the generating model, originating center, and associations among reports from the same case was concealed from the reviewer. A breast pathologist independently evaluated all reports in three sessions conducted at intervals of one week. The report order was randomly shuffled, and new anonymous identifiers were assigned before each session. Ratings from previous sessions were not available during subsequent assessments. The pathologist applied the same rubric that prioritized diagnostic honesty and used the same scoring range from 0 to 1 as the DeepSeek API evaluator. The mean score across the three independent review sessions was used as the pathologist reference score.

\subsection{Definitions of Automated Text-Similarity Metrics}
\label{app:metrics}
We provide the formal definitions of the automatic metrics used in our quantitative evaluation.

\paragraph{BLEU.}
The BLEU score is computed as:
\begin{equation}
    \text{BLEU} = \text{BP} \cdot \exp\left( \sum_{n=1}^{N} w_n \log p_n \right),
\end{equation}
where $p_n$ denotes the modified $n$-gram precision for order $n$, $w_n$ is the weight for each $n$-gram order (typically $w_n = 1/N$), and $N$ is the maximum $n$-gram order (commonly $N=4$). The brevity penalty $\text{BP}$ is defined as:
\begin{equation}
    \text{BP} = 
    \begin{cases}
        1 & \text{if } c > r, \\
        \exp(1 - r/c) & \text{if } c \leq r,
    \end{cases}
\end{equation}
where $c$ is the length of the candidate generation and $r$ is the effective reference length.

\paragraph{ROUGE.}
ROUGE-$N$ measures the recall of $n$-gram overlap between the generated text and reference:
\begin{equation}
    \text{ROUGE-}N = \frac{\sum_{S \in \mathcal{R}} \sum_{\text{gram}_n \in S} \text{Count}_{\text{match}}(\text{gram}_n)}{\sum_{S \in \mathcal{R}} \sum_{\text{gram}_n \in S} \text{Count}(\text{gram}_n)},
\end{equation}
where $\mathcal{R}$ denotes the set of reference reports, $\text{Count}(\text{gram}_n)$ is the total occurrence of an $n$-gram in the references, and $\text{Count}_{\text{match}}(\text{gram}_n)$ is the maximum number of matching $n$-grams between the candidate and references. ROUGE-L, based on the longest common subsequence (LCS), is also commonly reported to capture sentence-level structural similarity.

\paragraph{METEOR.}
METEOR computes a harmonic mean of unigram precision and recall, adjusted by a fragmentation penalty:
\begin{equation}
    \text{METEOR} = F_{\text{mean}} \cdot (1 - \text{Pen}),
\end{equation}
where
\begin{equation}
    F_{\text{mean}} = \frac{P \cdot R}{\alpha P + (1 - \alpha) R},
\end{equation}
with $P$ and $R$ denoting unigram precision and recall, respectively, and $\alpha$ is a weighting parameter (typically $\alpha = 0.9$). The fragmentation penalty $\text{Pen}$ is defined as:
\begin{equation}
    \text{Pen} = \gamma \cdot \left( \frac{\text{ch}}{m} \right)^\theta,
\end{equation}
where $\text{ch}$ is the number of chunks (contiguous matched sequences), $m$ is the total number of matched unigrams, and $\gamma, \theta$ are penalty parameters. METEOR additionally incorporates synonym matching and stemming to capture semantic equivalence beyond exact string matches.

% \subsection{Code and Data Availability}
% The private institutional WSIs and diagnostic reports used for model adaptation and validation are not publicly available due to patient privacy, institutional data governance, and ethical restrictions. Public benchmark datasets analyzed in this study are publicly available from their original repositories, including BCNB (https://bcnb.grand-challenge.org/) and BRACS (https://www.bracs.icar.cnr.it/download/). The pretrained weights of the baseline models are publicly available from Hugging Face, including PRISM (https://huggingface.co/paige-ai/Prism) and TITAN (https://huggingface.co/MahmoodLab/TITAN). The preprocessing, training, evaluation, statistical analysis, and risk-control code, together with the trained CorePath model weights, will be made publicly available (https://github.com/danninglee/CorePath) upon acceptance of the manuscript to support reproducibility. Qualified academic researchers with approved access to PRISM may reproduce and use the adaptation pipeline under the applicable license terms.

\newpage
\section{Supplementary Results}

\subsection{Supplementary Figures}
\renewcommand{\thefigure}{S\arabic{figure}}

\begin{figure}[htbp]
\centering

\begin{subfigure}[t]{0.49\textwidth}
    \centering
    \includegraphics[width=\linewidth]{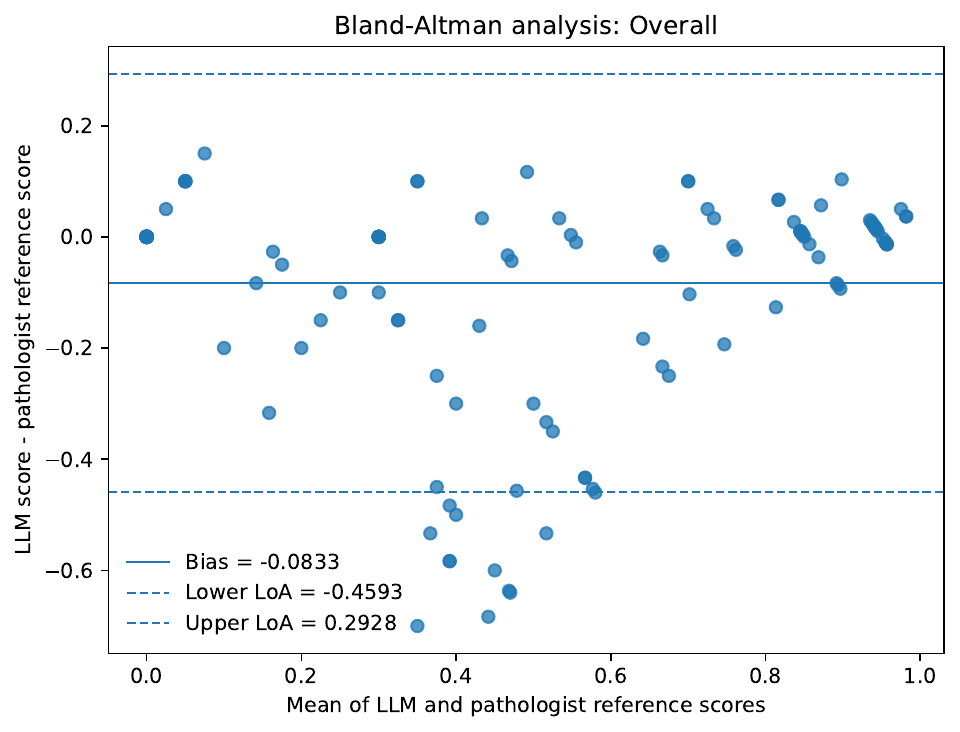}
    \caption{}
    \label{fig:bland_altman_overall}
\end{subfigure}
\hfill
\begin{subfigure}[t]{0.49\textwidth}
    \centering
    \includegraphics[width=\linewidth]{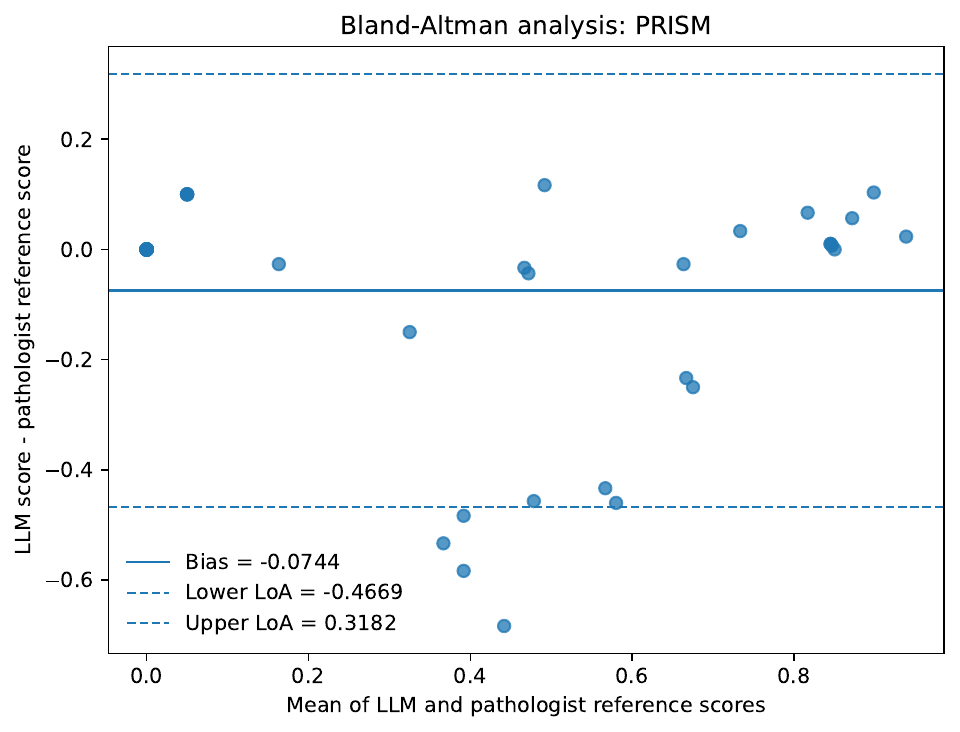}
    \caption{}
    \label{fig:bland_altman_prism}
\end{subfigure}

\vspace{0.5em}

\begin{subfigure}[t]{0.49\textwidth}
    \centering
    \includegraphics[width=\linewidth]{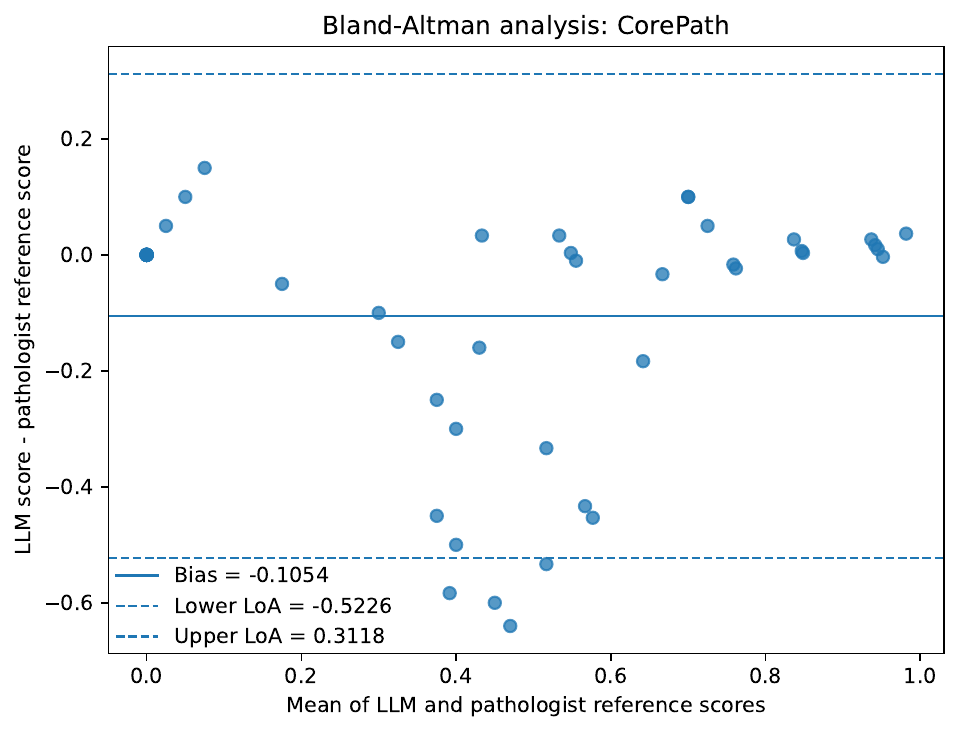}
    \caption{}
    \label{fig:bland_altman_corepath}
\end{subfigure}
\hfill
\begin{subfigure}[t]{0.49\textwidth}
    \centering
    \includegraphics[width=\linewidth]{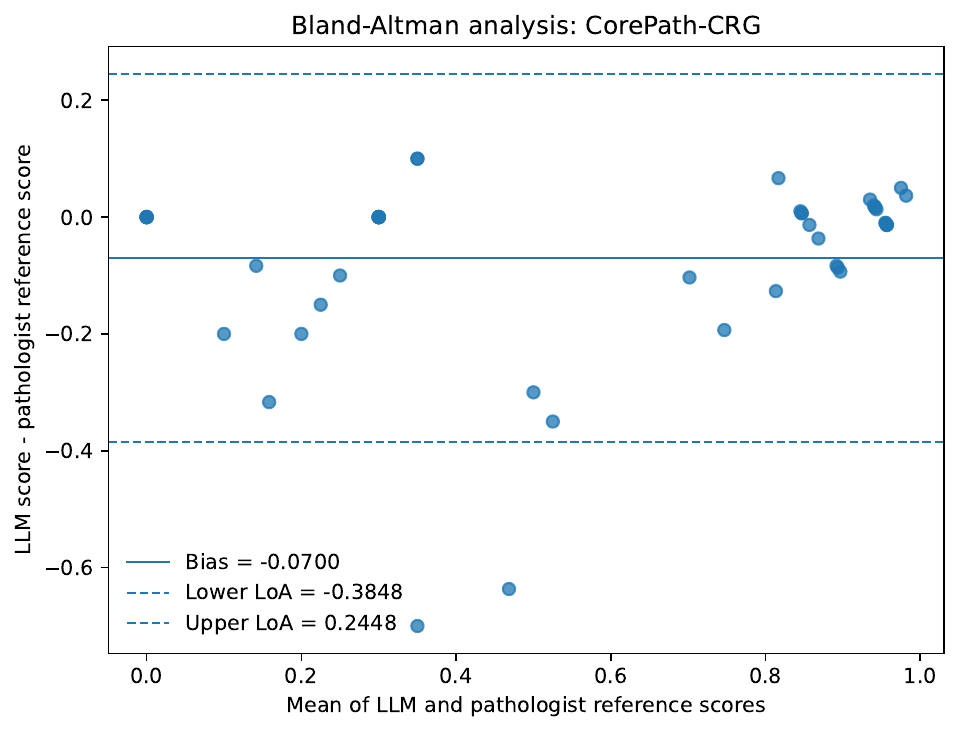}
    \caption{}
    \label{fig:bland_altman_corepath_crg}
\end{subfigure}

\caption{
\textbf{Bland-Altman analysis of agreement between LLM-based Evaluation Scores and pathologist reference scores.}
Panels show the results for (a) all reports ($N=144$), (b) reports generated by PRISM ($N=48$), (c) reports generated by CorePath ($N=48$), and (d) reports generated by CorePath-CRG ($N=48$). For each report, the pathologist reference score was defined as the mean of three repeated assessments. The horizontal axis represents the mean of the LLM-based Evaluation Score and the pathologist reference score, whereas the vertical axis represents their difference, calculated as the LLM-based Evaluation Score minus the pathologist reference score. The solid horizontal line indicates the mean difference (bias), and the upper and lower dashed horizontal lines indicate the upper and lower 95\% limits of agreement, respectively, calculated as the bias $\pm 1.96$ times the standard deviation of the paired differences. The bias and 95\% limits of agreement are annotated in each panel. LLM, large language model; LoA, limits of agreement; $N$, the number of reports. The results show that the LLM-based Evaluation Scores align with pathologist reference scores but are generally more conservative.
}
\label{fig:bland_altman_agreement}

\end{figure}

\begin{figure}[htbp]
    \centering
    \includegraphics[width=.9\linewidth,height=0.88\textheight,keepaspectratio]{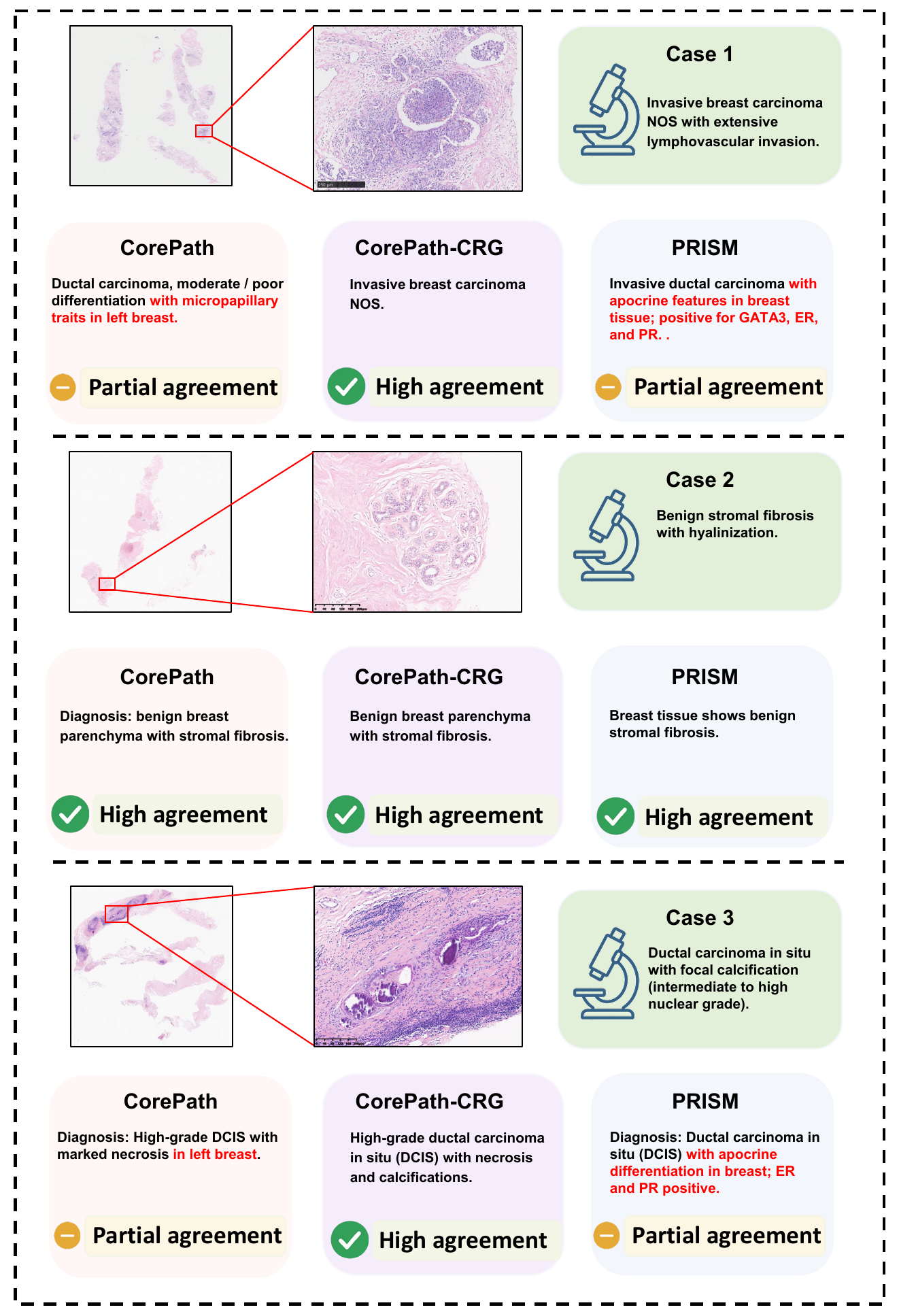}
    \caption{\textbf{Supplementary report-generation examples across different models.}
For each case, the left panels show the whole-slide image and a magnified region of interest, and the green panel summarizes the reference diagnosis. Model-generated outputs are shown below each case. Red text indicates an incorrect or unsupported diagnosis. Scale bars: from top to bottom, 250$\mu$m, 200$\mu$m, and 200$\mu$m.}
    \label{fig:report_examples_supp}
\end{figure}

\newpage
\subsection{Supplementary Tables}
% ── Table S1: SPH-2 breast biopsy dataset ──────────────────────────
\begin{table}[htbp]
\centering
\normalsize
\renewcommand{\arraystretch}{1.15}
\setlength{\tabcolsep}{6pt}
\caption{Information on different tasks about the SPH-2 breast biopsy dataset.}
\label{tab:diagnostic_task_summary_SPH-2}
\begin{tabular}{l c l c}
\toprule
Task  & Categories & Slides per class \\
\midrule
Cancer detection
& Cancer/noncancer
& 367/245 \\
Tumor category classification
& Noncancer/in situ/invasive
& 245/35/332 \\
Five-class histological subtype
& Noncancer/ILC/IDC/other IBC/in situ
& 245/20/282/30/35 \\
\midrule
Total
& 612
& 
&  \\
\bottomrule
\end{tabular}

\vspace{0.5em}
\begin{minipage}{0.95\linewidth}
\normalsize
\textit{Note.} In situ, carcinoma in situ; invasive, invasive breast carcinoma; ILC, invasive lobular carcinoma; IDC, invasive ductal carcinoma; other IBC, other invasive breast carcinoma.
\end{minipage}
\end{table}
 
\vspace{1em}
 
% ── Table S2: SWH breast biopsy dataset ──────────────────────────
\begin{table}[htbp]
\centering
\normalsize
\renewcommand{\arraystretch}{1.15}
\setlength{\tabcolsep}{6pt}
\caption{Information on different tasks about the SWH breast biopsy dataset.}
\label{tab:diagnostic_task_summary_SWH}
\begin{tabular}{l c l c}
\toprule
Task  & Categories & Slides per class \\
\midrule
Cancer detection
& Cancer/noncancer
& 585/807 \\
Tumor category classification
& Noncancer/in situ/invasive
& 807/44/541 \\
Five-class histological subtype
& Noncancer/ILC/IDC/other IBC/in situ
& 807/8/508/25/44 \\
\midrule
Total
& 1392
& 
&  \\
\bottomrule
\end{tabular}

\vspace{0.5em}
\begin{minipage}{0.95\linewidth}
\normalsize
\textit{Note.} In situ, carcinoma in situ; invasive, invasive breast carcinoma; ILC, invasive lobular carcinoma; IDC, invasive ductal carcinoma; other IBC, other invasive breast carcinoma.
\end{minipage}

\end{table}
 
\vspace{1em}
 
% ── Table S1: WCH-2 breast biopsy dataset ──────────────────────────
\begin{table}[htbp]
\centering
\normalsize
\renewcommand{\arraystretch}{1.15}
\setlength{\tabcolsep}{6pt}
\caption{Information on different tasks about the WCH-2 breast biopsy dataset.}
\label{tab:diagnostic_task_summary_WCH-2}
\begin{tabular}{l c l c}
\toprule
Task  & Categories & Slides per class \\
\midrule
Cancer detection
& Cancer/noncancer
& 336/160 \\
Tumor category classification
& Noncancer/in situ/invasive
& 160/48/288 \\
Five-class histological subtype
& Noncancer/ILC/IDC/other IBC/in situ
& 160/11/270/7/48 \\
\midrule
Total
& 496
& 
&  \\
\bottomrule
\end{tabular}

\vspace{0.5em}
\begin{minipage}{0.95\linewidth}
\normalsize
\textit{Note.} In situ, carcinoma in situ; invasive, invasive breast carcinoma; ILC, invasive lobular carcinoma; IDC, invasive ductal carcinoma; other IBC, other invasive breast carcinoma.
\end{minipage}

\end{table}
 
\vspace{1em}
  
% ── Table S4: WTH breast biopsy dataset ──────────────────────────
\begin{table}[htbp]
\centering
\normalsize
\renewcommand{\arraystretch}{1.15}
\setlength{\tabcolsep}{6pt}
\caption{Information on different tasks about the WTH breast biopsy dataset.}
\label{tab:diagnostic_task_summary_WTH}
\begin{tabular}{l c l c}
\toprule
Task  & Categories & Slides per class \\
\midrule
Cancer detection
& Cancer/noncancer
& 207/75 \\
Tumor category classification
& Noncancer/in situ/invasive
& 75/17/190 \\
Five-class histological subtype
& Noncancer/ILC/IDC/other IBC/in situ
& 75/4/185/1/17 \\
\midrule
Total
& 282
& 
&  \\
\bottomrule
\end{tabular}

\vspace{0.5em}
\begin{minipage}{0.95\linewidth}
\normalsize
\textit{Note.} In situ, carcinoma in situ; invasive, invasive breast carcinoma; ILC, invasive lobular carcinoma; IDC, invasive ductal carcinoma; other IBC, other invasive breast carcinoma.
\end{minipage}

\end{table}
 
\vspace{1em}

% ── Table S5: SJH breast biopsy dataset ──────────────────────────
\begin{table}[htbp]
\centering
\normalsize
\renewcommand{\arraystretch}{1.15}
\setlength{\tabcolsep}{6pt}
\caption{Information on different tasks about the SJH breast biopsy dataset.}
\label{tab:diagnostic_task_summary_SJH}
\begin{tabular}{l c l c}
\toprule
Task  & Categories & Slides per class \\
\midrule
Cancer detection
& Cancer/noncancer
& 76/52 \\
Tumor category classification
& Noncancer/in situ/invasive
& 52/6/70 \\
Five-class histological subtype
& Noncancer/ILC/IDC/other IBC/in situ
& 52/2/67/1/6 \\
\midrule
Total
& 128
& 
&  \\
\bottomrule
\end{tabular}

\vspace{0.5em}
\begin{minipage}{0.95\linewidth}
\normalsize
\textit{Note.} In situ, carcinoma in situ; invasive, invasive breast carcinoma; ILC, invasive lobular carcinoma; IDC, invasive ductal carcinoma; other IBC, other invasive breast carcinoma.
\end{minipage}

\end{table}
 
\vspace{1em}

% ── Table S6: SZH breast biopsy dataset ──────────────────────────
\begin{table}[htbp]
\centering
\normalsize
\renewcommand{\arraystretch}{1.15}
\setlength{\tabcolsep}{6pt}
\caption{Information on different tasks about the SZH breast biopsy dataset.}
\label{tab:diagnostic_task_summary_SZH}
\begin{tabular}{l c l c}
\toprule
Task  & Categories & Slides per class \\
\midrule
Cancer detection
& Cancer/noncancer
& 75/60 \\
Tumor category classification
& Noncancer/in situ/invasive
& 60/3/72 \\
\midrule
Total
& 135
& 
&  \\
\bottomrule
\end{tabular}

\vspace{0.5em}
\begin{minipage}{0.95\linewidth}
\normalsize
\textit{Note.} In situ, carcinoma in situ; invasive, invasive breast carcinoma.
\end{minipage}

\end{table}
\FloatBarrier
% Table~\ref{tab:zeroshot_binary} to \ref{tab:zeroshot_open} show numerical results for zero-shot classifications.

% ── Table S5: PROMPT for zero-shot 2class ───────────────────────
\begin{table}[ht]
\centering
\normalsize
\renewcommand{\arraystretch}{1.15}
\setlength{\tabcolsep}{6pt}
\caption{Prompt class names used for cancer detection.}
\label{tab:binary_prompt_class_names}
\begin{tabular}{l c l}
\toprule
Task & Class & Class names \\
\midrule
\multirow{4}{*}{Cancer detection}
& \multirow{2}{*}{Noncancer}
& noncancer \\
& & benign \\
\cmidrule(lr){2-3}
& \multirow{2}{*}{Cancer}
& cancer \\
& & carcinoma \\
\bottomrule
\end{tabular}
\end{table}

\vspace{1em}

% ── Table S6: PROMPT for zero-shot 3class ───────────────────────
\begin{table}[ht]
\centering
\normalsize
\renewcommand{\arraystretch}{1.15}
\setlength{\tabcolsep}{6pt}
\caption{Prompt class names used for invasion assessment.}
\label{tab:three_class_prompt_class_names}
\begin{tabular}{l l l}
\toprule
Task & Class & Class names \\
\midrule
\multirow{25}{2.5cm}{Invasion assessment}
& \multirow{2}{*}{Noncancer}
& noncancer \\
& & benign \\
\cmidrule(lr){2-3}
& \multirow{21}{2cm}{Invasive breast carcinoma}
& invasive lobular carcinoma \\
& & breast invasive lobular carcinoma \\
& & invasive lobular carcinoma of the breast \\
& & invasive carcinoma of the breast, lobular pattern \\
& & breast ILC \\
& & invasive ductal carcinoma \\
& & breast invasive ductal carcinoma \\
& & invasive ductal carcinoma of the breast \\
& & invasive carcinoma of the breast, ductal pattern \\
& & breast IDC \\
& & tubular carcinoma \\
& & cribriform carcinoma \\
& & mucinous carcinoma \\
& & mucinous cystadenocarcinoma \\
& & invasive micropapillary carcinoma \\
& & carcinoma with apocrine differentiation \\
& & metaplastic carcinoma \\
& & rare and salivary gland-type tumours \\
& & neuroendocrine neoplasms \\
& & neuroendocrine tumour \\
& & neuroendocrine carcinoma \\
\cmidrule(lr){2-3}
& \multirow{2}{2cm}{In situ carcinoma}
& lobular carcinoma in situ \\
& & ductal carcinoma in situ \\
\bottomrule
\end{tabular}
\end{table}

\vspace{1em}

% ── Table S7: PROMPT for zero-shot 5class ───────────────────────
\begin{table}[ht]
\centering
\normalsize
\renewcommand{\arraystretch}{1.15}
\setlength{\tabcolsep}{6pt}
\caption{Prompt class names used for histological subtyping.}
\label{tab:five_class_prompt_class_names}
\begin{tabular}{l l l}
\toprule
Task & Class & Class names \\
\midrule
\multirow{25}{2.5cm}{Histological subtyping}
& \multirow{2}{2cm}{Noncancer}
& noncancer \\
& & benign \\
\cmidrule(lr){2-3}
& \multirow{2}{2cm}{In situ carcinoma}
& lobular carcinoma in situ \\
& & ductal carcinoma in situ \\
\cmidrule(lr){2-3}
& \multirow{5}{2cm}{Invasive breast carcinoma NOS}
& invasive ductal carcinoma \\
& & breast invasive ductal carcinoma \\
& & invasive ductal carcinoma of the breast \\
& & invasive carcinoma of the breast, ductal pattern \\
& & breast IDC \\
\cmidrule(lr){2-3}
& \multirow{5}{2cm}{Invasive lobular carcinoma}
& invasive lobular carcinoma \\
& & breast invasive lobular carcinoma \\
& & invasive lobular carcinoma of the breast \\
& & invasive carcinoma of the breast, lobular pattern \\
& & breast ILC \\
\cmidrule(lr){2-3}
& \multirow{11}{2cm}{Other invasive breast carcinoma}
& tubular carcinoma \\
& & cribriform carcinoma \\
& & mucinous carcinoma \\
& & mucinous cystadenocarcinoma \\
& & invasive micropapillary carcinoma \\
& & carcinoma with apocrine differentiation \\
& & metaplastic carcinoma \\
& & rare and salivary gland-type tumours \\
& & neuroendocrine neoplasms \\
& & neuroendocrine tumour \\
& & neuroendocrine carcinoma \\
\bottomrule
\end{tabular}

\end{table}

\vspace{1em}

% ── Table S8: PROMPT for zero-shot 3class BCNB ───────────────────────
\begin{table}[ht]
\centering
\normalsize
\renewcommand{\arraystretch}{1.15}
\setlength{\tabcolsep}{6pt}
\caption{Prompt class names used for invasive carcinoma subtype categorization of BCNB.}
\label{tab:invasive_three_class_prompt_class_names}
\begin{tabular}{l l l}
\toprule
Task & Class & Class names \\
\midrule
\multirow{21}{2.5cm}{Invasive carcinoma subtype categorization}
& \multirow{5}{2cm}{Invasive breast carcinoma NOS}
& invasive ductal carcinoma \\
& & breast invasive ductal carcinoma \\
& & invasive ductal carcinoma of the breast \\
& & invasive carcinoma of the breast, ductal pattern \\
& & breast IDC \\
\cmidrule(lr){2-3}
& \multirow{5}{2cm}{Invasive lobular carcinoma}
& invasive lobular carcinoma \\
& & breast invasive lobular carcinoma \\
& & invasive lobular carcinoma of the breast \\
& & invasive carcinoma of the breast, lobular pattern \\
& & breast ILC \\
\cmidrule(lr){2-3}
& \multirow{11}{2cm}{Other invasive breast carcinoma}
& tubular carcinoma \\
& & cribriform carcinoma \\
& & mucinous carcinoma \\
& & mucinous cystadenocarcinoma \\
& & invasive micropapillary carcinoma \\
& & carcinoma with apocrine differentiation \\
& & metaplastic carcinoma \\
& & rare and salivary gland-type tumours \\
& & neuroendocrine neoplasms \\
& & neuroendocrine tumour \\
& & neuroendocrine carcinoma \\
\bottomrule
\end{tabular}
\end{table}

\vspace{1em}

% ── Table S9: PROMPT for zero-shot 3class BRACS ───────────────────────
\begin{table}[ht]
\centering
\normalsize
\renewcommand{\arraystretch}{1.15}
\setlength{\tabcolsep}{6pt}
\caption{Prompt class names used for lesion stratification of BRACS.}
\label{tab:bracs_three_class_prompt_class_names}
\begin{tabular}{l l l}
\toprule
Task & Class & Class names \\
\midrule
\multirow{13}{3.5cm}{Lesion stratification}
& \multirow{4}{*}{Benign tumors}
& benign breast tumor \\
& & benign breast lesion \\
& & benign breast neoplasm \\
& & non-malignant breast tumor \\
\cmidrule(lr){2-3}
& \multirow{4}{*}{Atypical tumors}
& atypical breast lesion \\
& & breast atypia \\
& & atypical proliferative lesion of breast \\
& & premalignant breast lesion \\
\cmidrule(lr){2-3}
& \multirow{5}{*}{Malignant tumors}
& malignant breast tumor \\
& & breast cancer \\
& & breast carcinoma \\
& & breast malignancy \\
& & malignant breast neoplasm \\
\bottomrule
\end{tabular}
\end{table}

\vspace{1em}

% ── Table S10: PROMPT for zero-shot 7class BRACS ───────────────────────
\begin{table}[ht]
\centering
\normalsize
\renewcommand{\arraystretch}{1.15}
\setlength{\tabcolsep}{6pt}
%\scriptsize
\caption{Prompt class names used for fine-grained classification of BRACS.}
\label{tab:bracs_seven_class_prompt_class_names}
\begin{tabular}{p{3.5cm} p{4.0cm} @{\hspace{0.65cm}} p{5cm}}
\toprule
Task & Class & Class names \\
\midrule
\multirow{36}{3.5cm}{Fine-grained classification}
& \multirow{4}{*}{Normal}
& normal breast tissue \\
& & normal mammary tissue \\
& & healthy breast tissue \\
& & non-lesional breast tissue \\
\cmidrule(lr){2-3}
& \multirow{4}{*}{Pathological benign}
& benign breast lesion \\
& & benign breast pathology \\
& & benign breast disease \\
& & non-malignant breast lesion \\
\cmidrule(lr){2-3}
& \multirow{5}{*}{Usual ductal hyperplasia}
& usual ductal hyperplasia \\
& & benign ductal hyperplasia \\
& & usual epithelial hyperplasia \\
& & proliferative lesion without atypia \\
& & UDH \\
\cmidrule(lr){2-3}
& \multirow{5}{*}{Flat epithelial atypia}
& flat epithelial atypia \\
& & flat epithelial lesion with atypia \\
& & columnar cell lesion with atypia \\
& & atypical columnar cell change \\
& & FEA \\
\cmidrule(lr){2-3}
& \multirow{6}{*}{Atypical ductal hyperplasia}
& atypical ductal hyperplasia \\
& & ductal hyperplasia with atypia \\
& & atypical intraductal proliferation \\
& & atypical hyperplasia of the breast \\
& & atypical ductal proliferation \\
& & ADH \\
\cmidrule(lr){2-3}
& \multirow{6}{*}{Ductal carcinoma in situ}
& ductal carcinoma in situ \\
& & in situ ductal carcinoma \\
& & intraductal carcinoma \\
& & non-invasive ductal carcinoma \\
& & ductal carcinoma (in situ) \\
& & DCIS \\
\cmidrule(lr){2-3}
& \multirow{6}{*}{Invasive carcinoma}
& invasive breast carcinoma \\
& & invasive breast cancer \\
& & infiltrating breast carcinoma \\
& & invasive malignant breast tumor \\
& & carcinoma of the breast, invasive \\
\bottomrule
\end{tabular}
\end{table}

\FloatBarrier

% ── Table 1: Binary Classification ──────────────────────────
\begin{table}[ht]
\centering
\normalsize
\renewcommand{\arraystretch}{1.15}
\setlength{\tabcolsep}{6pt}
\caption{Zero-shot cancer detection performance on independent private datasets. Values are original test-set point estimates with 95\% bootstrap confidence intervals in parentheses. Bold indicates the best result per dataset per metric.}
\label{tab:zeroshot_binary}
\begin{tabular}{llccc}
\toprule
Dataset & Model & Accuracy & Weighted F1 & Weighted AUC \\
\midrule
\multirow{3}{*}{SPH-2}
    & PRISM              & 0.8758 (0.8480-0.9020)          & 0.8717 (0.8426-0.8989)          & 0.9572 (0.9384-0.9737)          \\
    & TITAN              & 0.9183 (0.8954-0.9395)        & 0.9173 (0.8940-0.9390)        & 0.9669 (0.9495-0.9810)          \\
    & CorePath    & \best{0.9330} (0.9134-0.9526)        & \best{0.9322} (0.9119-0.9521)        & \best{0.9783} (0.9653-0.9891)   \\
\midrule
\multirow{3}{*}{SWH}
    & PRISM              & 0.8556 (0.8369-0.8736)        & 0.8561 (0.8374-0.8741)        & 0.9756 (0.9682-0.9827)          \\
    & TITAN              & 0.9016 (0.8865-0.9167)        & 0.9022 (0.8870-0.9170)        & 0.9886 (0.9839-0.9929)          \\
    & CorePath    & \best{0.9555} (0.9447-0.9655)        & \best{0.9557} (0.9450-0.9657)        & \best{0.9948} (0.9909-0.9977)   \\
\midrule
\multirow{3}{*}{WCH-2}
    & PRISM              & 0.9133 (0.8851-0.9375)        & 0.9112 (0.8817-0.9362)        & 0.9701 (0.9538-0.9830)          \\
    & TITAN              & 0.9254 (0.9012-0.9476)        & 0.9247 (0.8993-0.9471)        & 0.9734 (0.9594-0.9838)          \\
    & CorePath    & \best{0.9536} (0.9335-0.9718)        & \best{0.9532} (0.9327-0.9716)        & \best{0.9912} (0.9825-0.9971)   \\
\midrule
\multirow{3}{*}{WTH}
    & PRISM              & 0.9291 (0.9007-0.9574)        & 0.9262 (0.8942-0.9564)        & 0.9665 (0.9408-0.9874)          \\
    & TITAN              & 0.9326 (0.9007-0.9610)        & 0.9301 (0.8956-0.9597)        & 0.9858 (0.9714-0.9954)          \\
    & CorePath    & \best{0.9539} (0.9291-0.9752)        & \best{0.9527} (0.9260-0.9749)        & \best{0.9907} (0.9821-0.9968)   \\
\midrule
\multirow{3}{*}{SJH}
    & PRISM              & 0.8594 (0.7969-0.9141)        & 0.8570 (0.7905-0.9130)        & 0.9370 (0.8840-0.9762)          \\
    & TITAN              & \best{0.9375} (0.8982-0.9766)        & \best{0.9371} (0.8956-0.9765)        & 0.9603 (0.9167-0.9965)          \\
    & CorePath    & 0.9297 (0.8826-0.9688)        & 0.9293 (0.8789-0.9687)        & \best{0.9669} (0.9271-0.9931)   \\
\midrule
\multirow{3}{*}{SZH}
    & PRISM              & 0.8593 (0.7926-0.9111)        & 0.8539 (0.7850-0.9091)        & 0.9942 (0.9843-0.9998)          \\
    & TITAN              & \best{0.9630} (0.9259-0.9926)        & \best{0.9629} (0.9254-0.9926)        & 0.9891 (0.9668-1.0000)   \\
    & CorePath    & 0.9556 (0.9111-0.9852)        & 0.9552 (0.9115-0.9852)        & \best{0.9989} (0.9956-1.0000)          \\
\bottomrule
\end{tabular}
\end{table}
 
\vspace{1em}
 
% ── Table 2: Three-Class Classification ─────────────────────
\begin{table}[ht]
\centering
\normalsize
\renewcommand{\arraystretch}{1.15}
\setlength{\tabcolsep}{6pt}
\caption{Zero-shot invasion assessment performance on independent private datasets. Values are original test-set point estimates with 95\% bootstrap confidence intervals in parentheses. Bold indicates the best result per dataset per metric.}
\label{tab:zeroshot_3class}
\begin{tabular}{llccc}
\toprule
Dataset & Model & Accuracy & Weighted F1 & Weighted AUC \\
\midrule
\multirow{3}{*}{SPH-2}
    & PRISM              & 0.7010 (0.6634-0.7369)        & 0.6798 (0.6354-0.7192)        & 0.8705 (0.8448-0.8947)          \\
    & TITAN              & 0.5654 (0.5261-0.6046)        & 0.5004 (0.4568-0.5447)        & 0.9635 (0.9481-0.9753)   \\
    & CorePath    & \best{0.8268} (0.7941-0.8578)        & \best{0.8234} (0.7892-0.8552)        & \best{0.9643} (0.9487-0.9775)          \\
\midrule
\multirow{3}{*}{SWH}
    & PRISM              & 0.5999 (0.5754-0.6279)        & 0.6061 (0.5799-0.6365)        & 0.8425 (0.8219-0.8601)          \\
    & TITAN              & 0.3886 (0.3635-0.4145)        & 0.3249 (0.2989-0.3526)        & 0.9745 (0.9668-0.9815)          \\
    & CorePath    & \best{0.8355} (0.8168-0.8542)        & \best{0.8397} (0.8210-0.8582)        & \best{0.9860} (0.9800-0.9908)   \\
\midrule
\multirow{3}{*}{WCH-2}
    & PRISM              & 0.7460 (0.7056-0.7863)        & 0.7371 (0.6948-0.7787)        & 0.8683 (0.8377-0.8954)          \\
    & TITAN              & 0.6089 (0.5645-0.6512)        & 0.5371 (0.4897-0.5838)        & 0.9405 (0.9203-0.9594)          \\
    & CorePath    & \best{0.8710} (0.8407-0.8992)        & \best{0.8660} (0.8330-0.8953)        & \best{0.9793} (0.9680-0.9878)   \\
\midrule
\multirow{3}{*}{WTH}
    & PRISM              & 0.7730 (0.7234-0.8192)        & 0.7576 (0.7014-0.8103)        & 0.8758 (0.8354-0.9116)          \\
    & TITAN              & 0.7021 (0.6454-0.7589)        & 0.6443 (0.5818-0.7082)        & 0.9666 (0.9415-0.9856)          \\
    & CorePath    & \best{0.8972} (0.8617-0.9291)        & \best{0.8924} (0.8521-0.9271)        & \best{0.9840} (0.9703-0.9944)   \\
\midrule
\multirow{3}{*}{SJH}
    & PRISM              & 0.6797 (0.5938-0.7578)        & 0.6671 (0.5748-0.7530)        & 0.7718 (0.6892-0.8482)          \\
    & TITAN              & 0.5703 (0.4844-0.6563)        & 0.4894 (0.3957-0.5859)        & 0.9681 (0.9339-0.9928)          \\
    & CorePath    & \best{0.8125} (0.7500-0.8750)        & \best{0.8067} (0.7347-0.8760)        & \best{0.9804} (0.9561-0.9949)   \\
\midrule
\multirow{3}{*}{SZH}
    & PRISM              & 0.5704 (0.4815-0.6519)        & 0.5419 (0.4434-0.6337)        & 0.6866 (0.5992-0.7658)          \\
    & TITAN              & 0.1259 (0.0741-0.1852)        & 0.1658 (0.0976-0.2431)        & 0.6138 (0.5318-0.6923)          \\
    & CorePath    & \best{0.8074} (0.7333-0.8741)        & \best{0.7908} (0.7078-0.8604)        & \best{0.9881} (0.9740-0.9969)   \\
\bottomrule
\end{tabular}
\end{table}
 
\vspace{1em}
 
% ── Table 3: Five-Class Classification ──────────────────────
\begin{table}[ht]
\centering
\normalsize
\renewcommand{\arraystretch}{1.15}
\setlength{\tabcolsep}{6pt}
\caption{Zero-shot histological subtyping performance on independent private datasets. Values are original test-set point estimates with 95\% bootstrap confidence intervals in parentheses. Bold indicates the best result per dataset per metric.}
\label{tab:zeroshot_5class}
\begin{tabular}{llccc}
\toprule
Dataset & Model & Accuracy & Weighted F1 & Weighted AUC \\
\midrule
\multirow{3}{*}{SPH-2}
    & PRISM              & 0.5343 (0.4918-0.5719)        & 0.5379 (0.4925-0.5797)        & 0.8412 (0.8160-0.8644)          \\
    & TITAN              & 0.4706 (0.4314-0.5114)        & 0.3694 (0.3267-0.4144)        & 0.9312 (0.9152-0.9444)          \\
    & CorePath    & \best{0.7337} (0.6993-0.7680)        & \best{0.7285} (0.6904-0.7644)        & \best{0.9557} (0.9412-0.9695)   \\
\midrule
\multirow{3}{*}{SWH}
    & PRISM              & 0.3384 (0.3147-0.3628)        & 0.4124 (0.3845-0.4405)        & 0.7299 (0.7104-0.7477)          \\
    & TITAN              & 0.3534 (0.3290-0.3779)        & 0.2676 (0.2418-0.2937)        & 0.9573 (0.9480-0.9660)          \\
    & CorePath    & \best{0.7636} (0.7407-0.7852)        & \best{0.7834} (0.7599-0.8030)        & \best{0.9735} (0.9658-0.9801)   \\
\midrule
\multirow{3}{*}{WCH-2}
    & PRISM              & 0.4960 (0.4516-0.5383)        & 0.5379 (0.4894-0.5829)        & 0.7338 (0.6953-0.7663)          \\
    & TITAN              & 0.5202 (0.4738-0.5625)        & 0.4315 (0.3799-0.4776)        & 0.9029 (0.8767-0.9257)          \\
    & CorePath    & \best{0.7621} (0.7218-0.7984)        & \best{0.7405} (0.6959-0.7819)        & \best{0.9526} (0.9361-0.9675)   \\
\midrule
\multirow{3}{*}{WTH}
    & PRISM              & 0.4858 (0.4291-0.5426)        & 0.5834 (0.5217-0.6416)        & 0.7487 (0.6982-0.7906)          \\
    & TITAN              & 0.6525 (0.5957-0.7163)        & 0.5751 (0.5081-0.6482)        & 0.9364 (0.9066-0.9610)          \\
    & CorePath    & \best{0.8511} (0.8085-0.8901)        & \best{0.8531} (0.8046-0.8949)        & \best{0.9704} (0.9511-0.9848)   \\
\midrule
\multirow{3}{*}{SJH}
    & PRISM              & 0.4766 (0.3906-0.5703)        & 0.5763 (0.4845-0.6629)        & 0.7257 (0.6592-0.7892)          \\
    & TITAN              & 0.5078 (0.4219-0.6016)        & 0.3993 (0.3086-0.5020)        & 0.9377 (0.8979-0.9696)          \\
    & CorePath    & \best{0.7656} (0.6875-0.8359)        & \best{0.7545} (0.6720-0.8307)        & \best{0.9613} (0.9154-0.9905)   \\
\bottomrule
\end{tabular}
\end{table}
 
\vspace{1em}
 
% ── Table 4: Open Data Classification ───────────────────────
\begin{table}[ht]
\centering
\normalsize
\renewcommand{\arraystretch}{1.15}
\setlength{\tabcolsep}{4.5pt}
\caption{Zero-shot classification performance on public benchmark datasets. Values are original test-set point estimates with 95\% bootstrap confidence intervals in parentheses. Bold indicates the best result per dataset per metric.}
\label{tab:zeroshot_open}
\begin{tabular}{l l l c c c}
\toprule
Dataset & Task & Model & Accuracy & Weighted F1 & Weighted AUC \\
\midrule
\multirow{3}{*}{BCNB} 
& \multirow{3}{*}{\makecell[l]{Invasive\\carcinoma\\subtyping}}
    & PRISM    & 0.2760 (0.2476-0.3034)        & 0.3562 (0.3227-0.3881)        & 0.5748 (0.5214-0.6304)        \\
&   & TITAN    & 0.7968 (0.7741-0.8214)        & 0.8280 (0.8064-0.8497)        & 0.7666 (0.7141-0.8182)        \\
&   & CorePath & \best{0.8989} (0.8800-0.9159)        & \best{0.8919} (0.8698-0.9114)        & \best{0.7780} (0.7279-0.8257) \\
\midrule
\multirow{3}{*}{BRACS}
& \multirow{3}{*}{\makecell[l]{Lesion\\stratification}}
    & PRISM    & 0.5868 (0.5448-0.6271)        & 0.6313 (0.5918-0.6684)        & 0.8016 (0.7700-0.8295)        \\
&   & TITAN    & 0.2687 (0.2321-0.3053)        & 0.2164 (0.1788-0.2507)        & 0.7937 (0.7586-0.8237)        \\
&   & CorePath & \best{0.6362} (0.5941-0.6746)        & \best{0.6772} (0.6400-0.7111)        & \best{0.8178} (0.7856-0.8458) \\
\midrule
\multirow{3}{*}{BRACS}
& \multirow{3}{*}{\makecell[l]{Fine-grained\\classification}}
    & PRISM    & 0.4168 (0.3766-0.4570)        & 0.4268 (0.3848-0.4699)        & 0.8106 (0.7844-0.8347)        \\
&   & TITAN    & 0.3821 (0.3400-0.4241)        & 0.3491 (0.3061-0.3919)        & 0.8121 (0.7908-0.8314)        \\
&   & CorePath & \best{0.4442} (0.4022-0.4863)        & \best{0.4602} (0.4192-0.5016)        & \best{0.8252} (0.8011-0.8485) \\
\bottomrule
\end{tabular}
\end{table}

\vspace{1em}

\begin{table}[htbp]
\centering
\normalsize
\renewcommand{\arraystretch}{1.15}
\setlength{\tabcolsep}{6pt}
\caption{Intra-rater reliability of repeated pathologist report-quality assessments.}
\label{tab:human_reliability}
\begin{tabular}{l c c c}
\toprule
\textbf{Reliability measure} & \textbf{$N$ reports} & \textbf{ICC} & \textbf{95\% CI} \\
\midrule
Single assessment, ICC(A,1) & 144 & 0.980 & 0.97--0.99 \\
Mean of three assessments, ICC(A,3) & 144 & 0.993 & 0.99--1.00 \\
\bottomrule
\end{tabular}

\vspace{0.5em}
\begin{minipage}{0.95\linewidth}
\normalsize
\textit{Note.} ICCs were estimated using a two-way mixed-effects, absolute-agreement model. ICC(A,1) denotes the reliability of a single pathologist assessment, whereas ICC(A,3) denotes the reliability of the mean of three repeated assessments. $N$, number of reports; CI, confidence interval.
\end{minipage}
\end{table}

\begin{table}[htbp]
\centering
\normalsize
\renewcommand{\arraystretch}{1.15}
\setlength{\tabcolsep}{6pt}
\caption{Agreement between LLM-based Evaluation Scores and pathologist reference scores.}
\label{tab:llm_human_agreement}
\begin{tabular}{l c c c c c c}
\toprule
\textbf{Group} & \textbf{$N$} & \textbf{Spearman $\rho$} & \textbf{ICC(A,1), 95\% CI} & \textbf{MAD} & \textbf{Bias} & \textbf{95\% LoA} \\
\midrule
Overall & 144 & 0.911 & 0.843 (0.74--0.90) & 0.112 & $-0.083$ & $-0.459$ to 0.293 \\
PRISM & 48 & 0.918 & 0.828 (0.70--0.90) & 0.109 & $-0.074$ & $-0.467$ to 0.318 \\
CorePath & 48 & 0.850 & 0.778 (0.58--0.88) & 0.137 & $-0.105$ & $-0.523$ to 0.312 \\
CorePath-CRG & 48 & 0.907 & 0.887 (0.78--0.94) & 0.090 & \textbf{$-0.070$} & $-0.385$ to 0.245 \\
\bottomrule
\end{tabular}

\vspace{0.5em}
\begin{minipage}{0.98\linewidth}
\normalsize
\textit{Note.} The pathologist reference score for each report was defined as the mean of three repeated assessments. ICC(A,1) denotes the two-way mixed-effects, absolute-agreement, single-measure ICC between LLM-based Evaluation Score and the corresponding pathologist reference score. Spearman $\rho$ measures rank association; MAD, mean absolute difference; bias, mean signed difference (LLM-based Evaluation Score minus pathologist reference score); LoA, limits of agreement; $N$, number of reports; CI, confidence interval.
\end{minipage}
\end{table}

\begin{table}[htbp]
\centering
\normalsize
\renewcommand{\arraystretch}{1.15}
\setlength{\tabcolsep}{6pt}
\caption{LLM-based Evaluation Score across different datasets (weighted composite score, 0-1). Mean and median values are reported to reflect both average performance and distributional stability. Best results are in \textbf{bold}.}
\label{tab:llm_eval}
\begin{tabular}{l l c c}
\toprule
\textbf{Dataset} & \textbf{Model} & \textbf{Mean} & \textbf{Median} \\
\midrule
\multirow{3}{*}{SWH} 
  & PRISM  & 0.335 & 0.150 \\
  & CorePath     & 0.373 & \textbf{0.300} \\
  & CorePath-CRG  & \textbf{0.385} & 0.250 \\
\midrule
  \multirow{3}{*}{WCH-2} 
  & PRISM  & 0.250 & 0.100 \\
  & CorePath     & 0.319 & 0.250 \\
  & CorePath-CRG  & \textbf{0.523} & \textbf{0.620} \\
\midrule
\multirow{3}{*}{WTH} 
  & PRISM  & 0.224 & 0.075 \\
  & CorePath     & 0.328 & 0.200 \\
  & CorePath-CRG  & \textbf{0.640} & \textbf{0.910} \\
\midrule
\multirow{3}{*}{SJH} 
  & PRISM  & 0.229 & 0.025 \\
  & CorePath     & 0.277 & 0.125 \\
  & CorePath-CRG  & \textbf{0.529} & \textbf{0.620} \\
\bottomrule
\end{tabular}
\end{table}

\begin{table}[htbp]
\centering
\normalsize
\renewcommand{\arraystretch}{1.15}
\setlength{\tabcolsep}{4.5pt}
\caption{Quantitative evaluation of PRISM, CorePath, and CorePath-CRG using automatic text-similarity metrics on the \textbf{retained} samples (excluding ``Unknown'' rejections) across four datasets. $N$ denotes the number of valid samples after exclusion. Best results are in \textbf{bold}.}
\label{tab:hard_metrics_exclude}

\begin{tabular}{ll cccccccc}
\toprule
\textbf{Dataset} & \textbf{Model} & \textbf{B-1} & \textbf{B-2} & \textbf{B-3} & \textbf{B-4} & \textbf{R-1} & \textbf{R-2} & \textbf{R-L} & \textbf{METEOR} \\
\midrule

\multirow{3}{*}{SWH ($N$=1384)} 
  & PRISM & 0.1652 & 0.0542 & 0.0330 & 0.0251 & 0.1434 & 0.0220 & 0.1262 & 0.1700 \\
  & CorePath & \textbf{0.1810} & \textbf{0.0566} & \textbf{0.0353} & \textbf{0.0269} & 0.1670 & 0.0196 & 0.1327 & \textbf{0.1725} \\
  & CorePath-CRG & 0.0857 & 0.0441 & 0.0328 & 0.0242 & \textbf{0.1833} & \textbf{0.0771} & \textbf{0.1819} & 0.1210 \\

\midrule

\multirow{3}{*}{WCH-2 ($N$=496)} 
  & PRISM & 0.1510 & 0.0628 & 0.0405 & 0.0294 & 0.1659 & 0.0406 & 0.1507 & 0.1944 \\
  & CorePath & 0.1702 & 0.0717 & 0.0459 & 0.0321 & 0.2013 & 0.0507 & 0.1801 & \textbf{0.2313} \\
  & CorePath-CRG & \textbf{0.2183} & \textbf{0.0917} & \textbf{0.0713} & \textbf{0.0557} & \textbf{0.3083} & \textbf{0.0509} & \textbf{0.3059} & 0.1893 \\

\midrule

\multirow{3}{*}{WTH ($N$=282)} 
  & PRISM & 0.1648 & 0.0576 & 0.0351 & 0.0265 & 0.1449 & 0.0234 & 0.1357 & 0.2124 \\
  & CorePath & 0.1931 & 0.0665 & 0.0427 & 0.0306 & 0.2064 & 0.0341 & 0.1930 & \textbf{0.2597} \\
  & CorePath-CRG & \textbf{0.2893} & \textbf{0.0998} & \textbf{0.0751} & \textbf{0.0643} & \textbf{0.4077} & \textbf{0.0374} & \textbf{0.4075} & 0.2209 \\

\midrule

\multirow{3}{*}{SJH ($N$=127)} 
  & PRISM & 0.1488 & 0.0562 & 0.0382 & 0.0293 & 0.1251 & 0.0150 & 0.1148 & 0.1779 \\
  & CorePath & 0.1555 & 0.0482 & 0.0307 & 0.0232 & 0.1501 & 0.0170 & 0.1360 & 0.1925 \\
  & CorePath-CRG & \textbf{0.2562} & \textbf{0.1054} & \textbf{0.0769} & \textbf{0.0621} & \textbf{0.3309} & \textbf{0.0457} & \textbf{0.3309} & \textbf{0.2035} \\

\bottomrule
\end{tabular}

\end{table}

\begin{table}[htbp]
\centering
\normalsize
\renewcommand{\arraystretch}{1.15}
\setlength{\tabcolsep}{6pt}
\caption{Post-conformal-filtering diagnostic label distribution across four centers. Values are percentages.}
\label{tab:pathology_stats}
\begin{tabular}{l c c c c}
\toprule
\textbf{Label} & \textbf{SWH} & \textbf{WCH-2} & \textbf{WTH} & \textbf{SJH} \\
\midrule
\multicolumn{5}{l}{\textit{Binary Cancer Status}} \\
\quad Cancer & 43.82\% & 66.53\% & 76.60\% & 57.81\% \\
\quad Noncancer & 55.39\% & 33.47\% & 23.40\% & 40.62\% \\
\quad Unknown & 0.79\% & --- & --- & 1.56\% \\
\midrule
\multicolumn{5}{l}{\textit{Histological Subtype}} \\
\quad Noncancer & 17.46\% & 13.31\% & 8.51\% & 14.84\% \\
\quad In situ carcinoma & 0.22\% & 2.02\% & 1.77\% & --- \\
\quad Invasive breast carcinoma NOS & 40.95\% & 59.88\% & 68.09\% & 58.59\% \\
\quad Invasive lobular carcinoma & 0.07\% & 0.20\% & 0.35\% & --- \\
\quad Other invasive breast carcinoma & 0.36\% & 0.20\% & --- & --- \\
\quad Unknown & 40.95\% & 24.40\% & 21.28\% & 26.56\% \\
\bottomrule
\end{tabular}

\vspace{0.5em}
\begin{minipage}{0.95\linewidth}
\normalsize
\textit{Note.} The table summarizes the distribution of diagnostic labels after conformal filtering. ``Unknown'' denotes predictions that did not satisfy the conformal confidence criterion and were therefore not retained as confident diagnostic evidence. NOS, not otherwise specified.
\end{minipage}
\end{table}

\begin{table}[htbp]
\centering
\normalsize
\renewcommand{\arraystretch}{1.15}
\setlength{\tabcolsep}{6pt}
\caption{CorePath-CRG final output statistics across four centers. Values are percentages.}
\label{tab:risk_stats}
\begin{tabular}{l c c c c}
\toprule
\textbf{Metric} & \textbf{SWH} & \textbf{WCH-2} & \textbf{WTH} & \textbf{SJH} \\
\midrule
Full Rejection Rate & 0.57\% & 0.00\% & 0.00\% & 0.78\% \\
Narrative Rejection Rate & 88.58\% &  84.48\% & 87.94\% & 89.84\% \\
Diagnostic Fallback Rate & 88.00\% &  84.48\% & 87.94\% & 89.06\% \\
\bottomrule
\end{tabular}

\vspace{0.5em}
\begin{minipage}{0.95\linewidth}
\normalsize
\textit{Note.} Full Rejection Rate denotes the proportion of cases in which the narrative failed the release criterion and both the subtype and binary cancer-status predictions were rejected, resulting in an ``Unknown'' output and deferral for pathologist review. Narrative Rejection Rate denotes the proportion of cases in which the CorePath-generated narrative report failed the release criterion. Diagnostic Fallback Rate denotes the proportion of all cases in which the CorePath-generated narrative report was rejected, but a histological subtype or binary cancer status prediction meeting the confidence criterion was retained as the final diagnostic output.
\end{minipage}
\end{table}

\FloatBarrier
\end{document}